\documentclass[sigconf]{acmart}

\usepackage{graphicx} % Required for inserting images
\usepackage{hyperref}
\usepackage{xcolor} 
\usepackage[most]{tcolorbox}
\usepackage{subcaption}
\usepackage{amsmath}
\usepackage{comment}

\usepackage{amssymb}
\usepackage{float}
\usepackage{cleveref}
\usepackage{dblfloatfix}
\usepackage{multirow}
\usepackage{geometry} 
\usepackage{caption} 
\usepackage{colortbl} 
\usepackage{array}
\usepackage{booktabs} 
\usepackage[subtle]{savetrees}
\usepackage{enumitem}
\usepackage{colortbl} 
\usepackage{makecell}
\usepackage{pifont}
\newcommand{\cmark}{\textcolor{green!60!black}{\ding{51}}}
\newcommand{\xmark}{\textcolor{red!60!black}{\ding{55}}}
\usepackage{tcolorbox}
\tcbuselibrary{listings, skins}
\usepackage{listings}

\usepackage{subcaption} % Required in your preamble
\AtBeginDocument{%
  }

\theoremstyle{definition}

\theoremstyle{plain}
\crefname{assumption}{Assumption}{Assumptions}
\crefname{remark}{Remark}{Remarks}
\crefname{lemma}{Lemma}{Lemmas}
\crefname{proposition}{Proposition}{Propositions}

\setcopyright{acmlicensed}
\copyrightyear{2026}
\acmYear{2026}
\setcopyright{cc}
\setcctype{by}
\acmConference[KDD '26]{Proceedings of the 32nd ACM SIGKDD Conference on Knowledge Discovery and Data Mining V.2}{August 09--13, 2026}{Jeju Island, Republic of Korea}
\acmBooktitle{Proceedings of the 32nd ACM SIGKDD Conference on Knowledge Discovery and Data Mining V.2 (KDD '26), August 09--13, 2026, Jeju Island, Republic of Korea}
\acmDOI{10.1145/3770855.3818097}
\acmISBN{979-8-4007-2259-2/2026/08}

\begin{document}

\title{Formalizing and Mitigating Structural Distortion in LLM Attention for Graph Reasoning}


\author{Donald Loveland}
\authornote{Work partially completed during Amazon internship. Now at Snap Inc.}
\affiliation{%
  \institution{University of Michigan}
  \city{Ann Arbor}
  \state{MI}
  \country{USA}}
\email{dlovelan@umich.edu}

\author{Puja Trivedi}
\affiliation{%
  \institution{Amazon}
  \city{Palo Alto}
  \state{CA}
  \country{USA}}
\email{pujatriv@amazon.com}

\author{Ari Weinstein}
\authornote{Now at Ellison Institute of Technology Oxford}
\affiliation{%
  \institution{Amazon}
  \city{Palo Alto}
  \state{CA}
  \country{USA}}
\email{arielwei@amazon.com}

\author{Edward W Huang}
\affiliation{%
  \institution{Amazon}
  \city{Palo Alto}
  \state{CA}
  \country{USA}}
\email{ewhuang@amazon.com}

\author{Danai Koutra}
\affiliation{%
  \institution{University of Michigan}
  \city{Ann Arbor}
  \state{MI}
  \country{USA}}
\email{dkoutra@umich.edu}

\renewcommand{\shortauthors}{Donald Loveland, Puja Trivedi, Ari Weinstein, Edward W Huang, and Danai Koutra}

\begin{abstract}

Large Language Models (LLMs) have shown promise for reasoning over Text-Attributed Graphs (TAGs). However, applying LLMs to graphs requires linearizing their structure into sequences, introducing distortion rooted in the graph bandwidth problem. While this distortion has been shown to degrade performance, it is often attributed to prompt design or model scale, leaving the underlying mechanism unclear.
In this work, we show \textit{how} rotary positional embeddings turn graph linearization into bandwidth-dependent attention decay, suppressing attention between graph-adjacent nodes that are forced far apart in the serialized sequence. This shifts the focus of LLM-based graph reasoning from prompt engineering and scaling toward correcting attention misalignment.
Motivated by this analysis, we propose \textbf{G}raph-\textbf{a}ligned \textbf{L}anguage \textbf{A}ttention (\textbf{GaLA}), a lightweight, inference-time modification for LLMs. GaLA biases attention toward graph-adjacent nodes while preserving the LLM's sequential inductive biases. Across TAG benchmarks, GaLA improves performance with negligible overhead, demonstrating that distortion is a correctable bottleneck in LLM-based graph reasoning.
 
%Large Language Models (LLMs) have recently shown promise for directly reasoning over Text-Attributed Graphs (TAGs) by leveraging their vast pretrained knowledge. However, applying LLMs to graph data requires linearizing their structure into token sequences, introducing distortion rooted in the graph bandwidth problem. While this distortion is known to degrade performance, it is often attributed to prompt design or model scale, leaving the underlying mechanism unclear. In this work, we characterize \textit{how} performance degradation arises due to the LLM’s positional biases (e.g. Rotary Positional Embeddings), suppressing attention between graph-adjacent nodes when they are far apart in the sequence, even when sufficient semantic information is present. This insight shifts the focus of LLM-based graph reasoning from prompt engineering and scaling to intervening on the attention mechanism. Motivated by this, we propose \textbf{G}raph-\textbf{a}ligned  \textbf{L}anguage \textbf{A}ttention (\textbf{GaLA}), a lightweight, inference-time attention modification that softly injects graph structure into LLM. GaLA biases the LLM's attention toward the graph structure while preserving sequential inductive biases, effectively interpolating between LLMs and graph transformers without training. Across multiple TAG benchmarks, GaLA improves zero-shot performance with minimal overhead, demonstrating that improving attention alignment is valuable for LLM-based graph reasoning.
\end{abstract}





\maketitle

\section{Introduction}

Large Language Models (LLMs) have transformed natural language processing (NLP), achieving strong performance across sentiment analysis \citep{zhang2023sentimentanalysiseralarge}, summarization \citep{liu2024learningsummarizelargelanguage, Ravaut2023OnCU, Jiang2024TriSumLS}, and multi-step reasoning~\citep{Zhao_2024, YU2024100076, huang2024understandingplanningllmagents}. Much of this success, however, is predicated on benchmarks and applications that present information as linear sequences, aligning naturally with the models' pretrained assumptions~\citep{vaswani2023attentionneed, goldberg2015primerneuralnetworkmodels, ivanov2024aibenchmarksdatasetsllm}. In contrast, many real-world domains provide additional context through relationships between texts, forming text-attributed graphs (TAGs) \citep{seo2024unleashingpotentialtextattributedgraphs}. Some examples include web pages with hyperlinks~\citep{BRODER2000309}, co-purchase networks \citep{1167344}, and citation graphs \citep{newman_soc}, where the added structure provides useful signal for downstream prediction.

Motivated by the potential value of these relationships, recent work has adopted the \textit{LLM-as-Predictor} paradigm, where a graph is serialized into a prompt to enable reasoning over raw text~\citep{chen2024exploringpotentiallargelanguage, wang2025exploringgraphtaskspure, fatemi2023talklikegraphencoding, lei2025exploringpotentiallargelanguage,WangRPAKDD25}. This approach is conceptually appealing as it avoids task-specific graph encoders while leveraging the LLM's pretrained knowledge. However, empirical studies have shown that off-the-shelf LLMs often struggle to outperform simple structure-aware baselines in this setting~\citep{chen2024exploringpotentiallargelanguage, loveland2025glancecontextlearningleverage, wang2025exploringgraphtaskspure}. This raises a basic mechanistic question: why does graph information that is presented in the prompt fail to be used effectively by the model?
While prior work has found a link between these performance degradations and structural distortion, induced by linearizing the graph into a prompt \citep{pmlr-v48-niepert16, cuthill_mckee, RENDL2021105422}, the precise mechanism linking this distortion to model failure remains underexplored. 
Instead, degraded performance has largely been approached by addressing factors such as prompting, model scale, or supervised fine-tuning \citep{kazemnejad2023impactpositionalencodinglength, firooz2025lostindistanceimpactcontextualproximity, wang2025exploringgraphtaskspure}. We argue that while these factors are relevant, they primarily describe \textit{when} performance degrades, not \textit{why} LLMs struggle with graph reasoning. Without this insight, future LLM-as-Predictor methods are limited to indirect adjustments that manipulate input conditions, rather than resolving the underlying mismatch between graph topology and sequence-based attention.

Bridging this gap requires  characterizing how linearization interacts with the LLM. Specifically, we hypothesize that a key failure mode is rooted in the inductive biases of Rotary Positional Embeddings (RoPE) \citep{su2023roformerenhancedtransformerrotary, barbero2025roundroundgomakes}. When a graph neighborhood is serialized, RoPE attenuates attention between graph-adjacent but sequentially distant nodes, even when sufficient semantic information is present. \Cref{fig:teaser_fig} illustrates this intuition as graph neighbors that should interact strongly can be rotated far apart given their distance in the prompt. Importantly, this hypothesis reframes long-context degradation as connectivity distortion, shifting the focus from optimizing prompts or model scale to resolving a geometric misalignment between graph topology and attention. This motivates our central question:

\begin{figure}[t]
    \centering    \includegraphics[width=\linewidth]{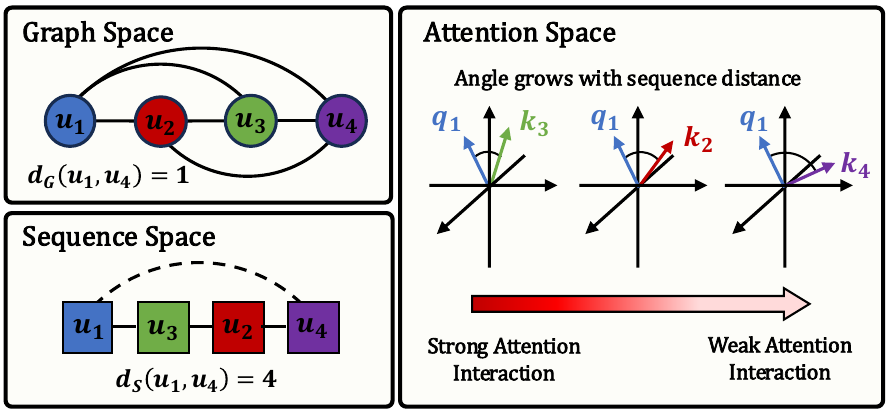}
    \Description{A visual diagram comparing three states: graph distance showing equidistant nodes, linearization showing nodes separated by gaps, and attention space showing these gaps converted into angles by RoPE.}
    \caption{Under graph distance, nodes are all equidistant. With linearization, graph-adjacent nodes are separated by positional gaps. In attention space, RoPE converts these gaps into angle differences, misrepresenting the underlying structure.} 
    \label{fig:teaser_fig}
\end{figure}

\begin{center}
\textbf{\textit{What encodes sequence-induced distortion in LLMs, and how can this finding encourage graph-aware LLM inference?}}
\end{center}
To address this, we first theoretically characterize how RoPE converts graph linearization into attention decay, and demonstrate how graph bandwidth can make this decay unavoidable. We validate this mechanism using a controlled aggregation task, where incorrect predictions correlate with attention attenuation. Furthermore, we demonstrate that strategies such as improved serialization, model scaling, and fine-tuning tend to succeed by counteracting this decay.

Building on these insights, we propose \textbf{G}raph-\textbf{a}ligned \textbf{L}anguage \textbf{A}ttention, \textbf{GaLA}, a lightweight, inference-time intervention that softly injects graph structure into the LLM's attention computation. Unlike methods that enforce hard graph constraints or introduce external graph encoders, GaLA applies a graph-based attention bias to frozen LLMs during inference. This intervention encourages attention between graph-adjacent nodes while preserving the model's pretrained sequential language-processing inductive biases. Across multiple TAG tasks, GaLA improves node classification with minimal overhead, improving both semi-synthetic and real-world graph reasoning by up to +18.6\% and +5.0\%, respectively. Together, our findings establish a new perspective on the LLM-as-Predictor paradigm focused on aligning the LLM's attention with the topology of the graph, rather than indirectly compensating for misalignment.

\begin{itemize}[leftmargin=*, labelsep=0.5em]
    \item \textbf{Formalizing Distortion.}  
    We formalize how linearization-based distortion modulates attention for graph inputs, suppressing interactions between graph-adjacent but sequence-distant nodes.

    \item \textbf{Limitations of Existing Strategies.}  
    We show that improved serialization, model scaling, and fine-tuning yield gains by mitigating attention attenuation, revealing a shared underlying mechanism.

\item \textbf{GaLA: Graph-aligned Language Attention.}
We propose GaLA, a soft, structure-aware attention modification that aligns LLM attention with graph topology while keeping model weights fixed. GaLA uses a cheap, one-time pre-inference calibration to select which heads receive the graph bias, either via a label-free entropy heuristic or, when calibration labels are available, a one-time gradient score. After this single step, the resulting head weights are fixed and GaLA runs entirely as an inference-time attention modification.

   \item \textbf{Improved LLM Performance.}  
Across semi-synthetic and real-world node classification tasks, GaLA yields gains of up to +18.6\% and +5.0\%, respectively, showing that inference-time attention alignment can improve graph reasoning with LLMs.
\end{itemize}

\section{Related Work} 

Recent work has explored approaches to allow LLMs to reason over TAGs. Broadly, these fall into: (i) prompting, (ii) training, (iii) hybrid systems and (iv) diagnostic analyses. We summarize these below. 

\vspace{0.1cm}
\noindent \textbf{Prompt-based Graph Reasoning.}
The LLM-as-Predictor paradigm serializes a node’s neighborhood into a prompt and queries an LLM for node or edge predictions \citep{chen2024exploringpotentiallargelanguage, wang2025exploringgraphtaskspure,He00CK25}. Prior work studies how to best expose structure via orderings and prompt templates, including neighborhood ordering heuristics \citep{chen2024llagalargelanguagegraph, xypolopoulos2025graphlinearizationmethodsreasoning} and structured prompting schemes such as GraphICL \citep{sun2025graphiclunlockinggraphlearning}. Complementary work analyzes how different graph-to-text encodings affect performance, showing large sensitivity to representation choice and task type \citep{fatemi2023talklikegraphencoding}. Subsequent work has focused on proposing serializations which are more faithful to the graph structure (e.g., tree-like encodings) for transformer reasoning \citep{zhao2023graphtextgraphreasoningtext}. While appealing for zero- or few-shot use, these methods remain fundamentally constrained by linearization: dense neighborhoods inevitably induce large separations between adjacent nodes, making outcomes brittle to ordering and contextual proximity \citep{firooz2025lostindistanceimpactcontextualproximity}.

\vspace{0.1cm}
\noindent \textbf{Training Adaptations.}
A second line of work improves structural reasoning by adapting pretrained LLMs using supervision or self-supervision on graph tasks \cite{he2024harnessingexplanationsllmtolminterpreter, zhang2025rethinkinggraphstructurelearning}. This includes instruction-tuning LLMs on natural-language graph descriptions \citep{ye2024languagegraphneeds, kim2023kggptgeneralframeworkreasoning}, graph-specific instruction tuning and alignment objectives \citep{tang2024graphgptgraphinstructiontuning}, and lightweight adapters or projection layers that map graph inputs into the LLM token space while often keeping most weights frozen \citep{chen2024llagalargelanguagegraph, hu2021loralowrankadaptationlarge}. These approaches can improve task performance, but require additional compute and may conflate genuine structural alignment with task-specific learning signals, making it less clear what mechanism in the pretrained model was limiting graph reasoning in the first place \citep{wang2025exploringgraphtaskspure}.

\vspace{0.1cm}
\noindent \textbf{Hybrid GNN-LLMs.}
Hybrid systems combine graph encoders with LLMs, using GNNs to compute structural representations that are injected via soft prompts, extra tokens, hidden-state conditioning, or retrieval pipelines \cite{tang2024graphgptgraphinstructiontuning, he2024harnessingexplanationsllmtolminterpreter, hu2025graggraphretrievalaugmentedgeneration,MaTK26}. Examples include using a GNN-derived prefix or soft prompt to steer a frozen LLM \citep{Liu_2024}, aligning graph-model embeddings to an LLM for open-ended prediction and explanation~\citep{zhang2024graphtranslatoraligninggraphmodel}, or using graph-aware retrieval to select in-context examples~\citep{hu2025letsaskgnnempowering}. These methods often perform well and scale better than raw serialization, but they effectively relocate structural reasoning to external components, reintroducing graph-specific inductive bias and reducing the appeal of LLM-native graph reasoning \citep{you2024largelanguagemodelsmeet}.

\vspace{0.1cm}
\noindent \textbf{Diagnostic Analyses.}
Separately, work on LLMs has found strong proximity biases, ordering sensitivities, and positional-encoding effects in long or non-linear contexts \citep{kazemnejad2023impactpositionalencodinglength, barbero2025roundroundgomakes}. In graph settings, multiple studies report that performance varies sharply with node ordering and contextual proximity \citep{chen2024llagalargelanguagegraph, sun2025graphiclunlockinggraphlearning, firooz2025lostindistanceimpactcontextualproximity}. Some of the most relevant work to ours is recent attention-focused analyses showing that LLM-based transformers may fail to capture inter-node relationships when structure is presented only through sequences \citep{guan2025attentionmechanismsperspectiveexploring}. These studies primarily characterize how attention distributions diverge from graph topology and propose training-time strategies to improve performance. 

\vspace{0.1cm}
\noindent \textbf{Our Work.}
In contrast to prior work, we investigate the underlying mechanism in LLMs by which linearization affects attention. We then demonstrate that this mechanism plays an important role in limiting reasoning, and use this understanding to develop a lightweight intervention that better aligns attention with graph topology while preserving the LLM’s sequential inductive biases.

\section{Preliminaries} %Definitions and Notation}

\textbf{Graphs.} Let $G=(V, E, X)$ be a TAG, where $V$ is the set of nodes, $E$ is the set of edges, and $X$ is the text for each node. Within each node's text attribute is a token sequence $x_i=[t_{i,1},t_{i,2},\dots,t_{i,l_i}]$ with length $l_i$. We denote $G$'s adjacency matrix as $\mathbf{A}\in\{0,1\}^{|V|\times |V|}$, where $\mathbf{A}_{ij}=1$ if $(v_i,v_j)\in E$. The $k$-hop ego-graph for a node $v$ is given by $\mathcal{N}_k(v)$, i.e., the subgraph containing all nodes and edges within $k$ hops of $v$. 

\vspace{0.1cm}
\noindent \textit{Graph Linearization.}
Linearization converts a graph into a token sequence through a mapping $S=L(G)$. This process first induces a node ordering $\pi_L:V\to\{1,\ldots,|V|\}$, after which node texts and structural markers are serialized into a string. The traversal strategy used to determine $\pi_L$ (e.g., BFS or DFS) is important for preserving graph relationships. For node classification, we generate a sequence $s_v$ for each target node $v$ by applying this linearization to its local $k$-hop neighborhood, i.e. $s_v=L(\mathcal{N}_k(v))$.

\vspace{0.1cm}
\noindent \textit{Edge Stretch Induced by Linearization.}
Linearization $L$ induces distortion by mapping graph-adjacent nodes to distant token positions. For each node $u$, let $\mathcal{T}_L(u)$ be the set of token positions occupied by $u$'s text in the serialized sequence, and let $p_u^L$ denote a representative position, such as the first token of $u$. We define the token-level edge stretch of $(u,v)\in E$ under $L$ as
\(
\delta^L_{u,v}=|p_u^L-p_v^L|.
\)
When the linearization is clearly stated, we write $\delta_{u,v}$ for brevity. Bandwidth theory implies that for non-trivial graph topologies, no ordering can preserve all adjacencies, making significant edge stretch unavoidable. Prior work typically attempts to mitigate this through ordering heuristics \citep{chen2024llagalargelanguagegraph, sun2025graphiclunlockinggraphlearning}, treating the LLM as a black box.

\vspace{0.1cm}
\noindent \textbf{Models. }Graph learning typically relies on graph neural networks (GNNs) \cite{kipf2017semisupervisedclassificationgraphconvolutional, hamilton2018inductiverepresentationlearninglarge}. GNNs obtain representations by aggregating information from the neighborhood around a node:

\vspace{-0.3cm}
{\small
\begin{equation*}
\mathbf{z}_i^{(k)} = \text{COMBINE}\left(\mathbf{z}_i^{(k-1)}, \text{AGGREGATE}\left(\{\mathbf{z}_j^{(k-1)} : v_j \in \mathcal{N}_1(v_i)\}\right)\right),
\end{equation*}
}
where $\mathbf{z}_{i}^{0}$ is an initial feature vector. For TAGs, this is typically a fixed text embedding derived from a method such as Word2Vec \cite{mikolov2013efficientestimationwordrepresentations} or Embedding LLMs \cite{zhang2025qwen3embeddingadvancingtext}. Importantly, the edges within each $\mathcal{N}_k(v)$ explicitly govern the flow of information. While computationally efficient, this design imposes rigid inductive biases \cite{loveland2023performancediscrepancieslocalhomophily,yan2022sidescoinheterophilyoversmoothing}, relies on fixed initial text embeddings, and cannot be applied zero-shot.
The \textit{LLM-as-Predictor} paradigm relaxes these constraints by directly leveraging frozen, pretrained LLMs \cite{wang2025exploringgraphtaskspure}. Instead of deriving embeddings via explicit message passing, this approach serializes the node's neighborhood into a prompt. Conceptually, this substitutes the GNN’s aggregation with the LLM’s self-attention, enabling flexible reasoning without the need for task-specific training.

\vspace{0.1cm}
\noindent \textbf{Positional Encodings.}
In decoder-based architectures, sequential order is captured via positional encodings. A common choice is RoPE, which applies position-dependent rotations to queries and keys~\citep{su2023roformerenhancedtransformerrotary, barbero2025roundroundgomakes}. Formally, for head dimension $d$ and $N=d/2$ two-dimensional subspaces, RoPE uses frequencies $\theta_m=b^{-m/N}=b^{-2m/d}$ for $m=0,\ldots,N-1$, where $b$ is the RoPE base. At sequence position $p$, the $m$-th subspace is rotated by angle $p\theta_m$.
For tokens at positions $p_i$ and $p_{i'}$, their relative rotation in subspace $m$ depends only on the positional gap $\Delta p=p_i-p_{i'}$:
{\small
\begin{equation*}
(\tilde{\mathbf{q}}^{(m)}_i)^\top \tilde{\mathbf{k}}^{(m)}_{i'}
=
(\mathbf{q}^{(m)}_i)^\top
\mathbf{R}_{\Delta p,\theta_m}
\mathbf{k}^{(m)}_{i'},
\qquad
\mathbf{R}_{\Delta p,\theta_m}
=
\mathbf{R}(\Delta p\,\theta_m).
\end{equation*}
}Equivalently, writing 
$\mathbf{q}^{(m)}_i=(q^{(m)}_{i,1},q^{(m)}_{i,2})$ and 
$\mathbf{k}^{(m)}_{i'}=(k^{(m)}_{i',1},k^{(m)}_{i',2})$, and letting 
$\phi_m=(p_{i'}-p_i)\theta_m$, the RoPE-modulated dot product is
{\small
\begin{equation*}
\begin{aligned}
(\tilde{\mathbf{q}}^{(m)}_i)^\top \tilde{\mathbf{k}}^{(m)}_{i'}
&=
\left(
q^{(m)}_{i,1}k^{(m)}_{i',1}
+
q^{(m)}_{i,2}k^{(m)}_{i',2}
\right)\cos(\phi_m) \\
&\quad+
\left(
q^{(m)}_{i,2}k^{(m)}_{i',1}
-
q^{(m)}_{i,1}k^{(m)}_{i',2}
\right)\sin(\phi_m).
\end{aligned}
\end{equation*}
}The first term is the usual query-key alignment modulated by relative position, while the second term captures the perpendicular interaction induced by the rotation. Thus, RoPE makes sequence distance directly affect attention scores through the phase $\phi_m$. When graphs are serialized, adjacent nodes can have large positional gaps $\delta^L_{u,v}$, causing graph-neighbor interactions to be attenuated by the model's sequence-based geometry. In this work, we characterize how this stretch propagates through RoPE and use the result to realign the model's internal geometry with the underlying graph structure.

\section{Mechanisms of Structural Distortion in LLMs}
\label{sec:distortion}

\begin{figure*}[t]
    \centering
    \includegraphics[width=0.86\linewidth]{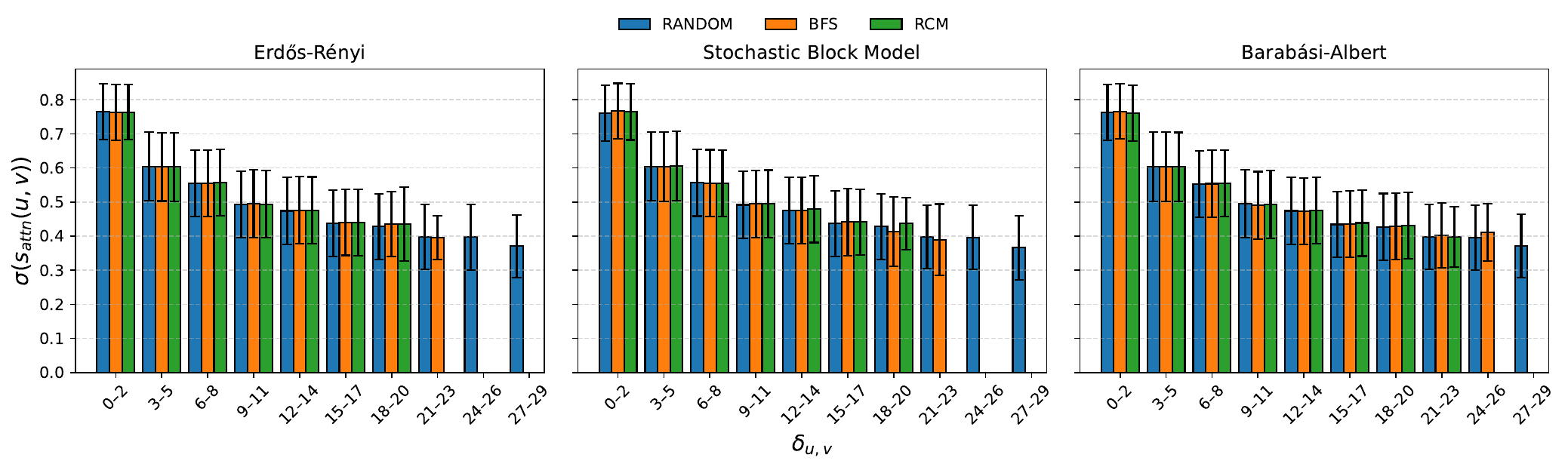}
    \vspace{-0.5cm}
    \Description{A wide, three-panel figure containing side-by-side bar charts for ER, SBM, and BA graphs. Each chart plots attention strength on the vertical axis against the serialization gap on the horizontal axis, with different colored bars representing various linearization methods. Across all three panels, the bars visually demonstrate a steep downward trend as the horizontal value increases.}
    \caption{Each plot reports the attention strength $ \sigma(s_{\mathrm{attn}}(u,v)) $ vs. serialization gap $\delta_{u,v}$ for ER, SBM, and BA graphs for different linearizations (bar colors). Attention decays sharply with increasing sequence gap across all graphs. Crucially, even RCM fails to prevent the large gaps that drive this decay, highlighting the intrinsic distortion of sequence-based serialization.}
    \label{fig:attn_sim}
    \vspace{-0.2cm}
\end{figure*}

In this section, we characterize the link between graph linearization and model behavior. We first define a node-level attention strength that measures how strongly two graph nodes interact inside the LLM's RoPE-modulated attention space. We then show that RoPE's frequency structure converts sequence distance into predictable decay in attention alignment. Finally, we connect this sequence-level decay to graph topology through graph bandwidth, yielding a topology-dependent bound on attention loss under serialization.

\subsection{Formalizing Alignment in Attention Space}

Sequence-level distortion metrics, such as the serialization gap $\delta_{u,v}$, capture how far graph-adjacent nodes are separated in the prompt. However, they do not directly describe how this separation affects the LLM's internal computation. To bridge this gap, we analyze alignment within the LLM's RoPE-modulated attention space. This allows us to characterize the model's internal behavior directly, rather than relying only on downstream performance metrics \citep{firooz2025lostindistanceimpactcontextualproximity, sun2025graphiclunlockinggraphlearning}.

\subsubsection{Evaluating Attention}
To analyze the LLM's internal processing, we first define a \textit{token-level attention score} between query and key tokens at positions $i$ and $j$, respectively. For a set of attention heads $\mathcal{H}$, let $\tilde{\mathbf{q}}^{(h)}_i$ and $\tilde{\mathbf{k}}^{(h)}_j$ denote the RoPE-modulated query and key vectors, and let $d$ denote the head dimension. The token-level score is the average pre-softmax logit across heads:
{\small
\begin{equation*}
s(i,j) \;=\; \frac{1}{|\mathcal{H}|} \sum_{h \in \mathcal{H}} \frac{\tilde{\mathbf{q}}^{(h)\top}_i\tilde{\mathbf{k}}^{(h)}_j}{\sqrt{d}}.
\end{equation*}
}
We then define \textit{Node Attention Similarity} by aggregating over node-specific tokens. Let $\mathcal{T}(u)$ and $\mathcal{T}(v)$ denote the token indices corresponding to the text of nodes $u$ and $v$. The node-level attention similarity is the mean score between all token pairs:

{\small
\begin{equation*}
s_{\mathrm{attn}}(u,v) \;=\; \frac{1}{|\mathcal{T}(u)|\,|\mathcal{T}(v)|} \sum_{i \in \mathcal{T}(u)} \sum_{j \in \mathcal{T}(v)} s(i,j).
\end{equation*}
}

\noindent \textbf{Intuition.} This metric serves as a proxy for the effective connectivity between nodes within the LLM's internal geometry. Since RoPE encodes position through rotations, large separations in the prompt induce phase shifts that systematically degrade dot-product alignment. As a result, $s_{\mathrm{attn}}(u,v)$ can decrease even when $u$ and $v$ are graph-adjacent and semantically related.

\subsection{Theoretical Characterization of RoPE's Impact on Attention under Linearization}
\label{sec:theory}

We now characterize how RoPE transforms sequence separation into attention decay. We begin with a pair of tokens separated by $\delta$ positions; later, for graph nodes $u$ and $v$, this token distance is instantiated by the edge stretch $\delta^L_{u,v}$ induced by linearization $L$. The node-level score above follows by averaging this effect over token pairs from two node texts. To isolate the role of positional encoding, we model pre-RoPE query and key representations as random variables with fixed semantic alignment, where semantic alignment captures similarity arising from node content before positional rotation.

\vspace{0.1cm}
\noindent \textbf{Setup.}
Consider a single layer and attention head with $N$ rotary subspaces. In the $m$-th subspace, RoPE uses frequency $\theta_m=b^{-m/N}$ for $m=0,\ldots,N-1$, where $b$ is the RoPE base. Let $\mathbf{q}^{(m)}_i,\mathbf{k}^{(m)}_j\in\mathbb{R}^2$ denote the pre-RoPE query and key components for tokens at positions $i$ and $j$. We define the pre-RoPE semantic alignment in subspace $m$ as
\[
\rho_m := \mathbb{E}\!\left[(\mathbf{q}^{(m)}_i)^\top \mathbf{k}^{(m)}_j\right],
\]
and let $\rho_{\mathrm{total}}:=\sum_{m=0}^{N-1}\rho_m$ denote the total expected pre-RoPE dot product across subspaces. We place no constraint on the in-subspace orientation of these representations; \Cref{app:thm1_proof} shows the rotation average depends only on $\rho_m$, up to a distance-independent skew constant.

RoPE rotates query and key vectors according to their positions. Therefore, the expected post-RoPE contribution of subspace $m$ depends on both the semantic alignment $\rho_m$ and the modulation $\cos(\delta\theta_m)$, where $\delta=|i-j|$. This gives the normalized discrete retention
\begin{equation}
r_{\mathrm{disc}}(\delta)
:=
\frac{\sum_{m=0}^{N-1}\rho_m \cos(\delta\theta_m)}
{\sum_{m=0}^{N-1}\rho_m}.
\label{eq:retention_discrete}
\end{equation}
In the continuum limit over rotary subspaces, the discrete frequencies $\{\theta_m\}_{m=0}^{N-1}$ are represented by a density $g(\theta)$, and the subspace alignments $\rho_m$ become a smooth function $\rho(\theta)$. Thus, the sums in \Cref{eq:retention_discrete} become integrals over the RoPE frequency range:
\begin{equation}
r(\delta)
=
\frac{\int_{1/b}^{1} \rho(\theta)g(\theta)\cos(\delta\theta)\,d\theta}
{\int_{1/b}^{1} \rho(\theta)g(\theta)\,d\theta}.
\label{eq:retention_general}
\end{equation}
Here, $g(\theta)$ is the architecture-induced frequency density, while $\rho(\theta)$ is the model- and input-dependent semantic alignment at $\theta$. Since RoPE uses $\theta_m=b^{-m/N}$, the change of variables $m=-N\log_b\theta$ gives
\begin{equation}
g_{\mathrm{RoPE}}(\theta)
=
\frac{1}{\theta\ln b},
\qquad
\theta\in[1/b,1].
\label{eq:rope_frequency_density}
\end{equation}

\begin{theorem}[Expected RoPE Attention under Edge Stretch]
\label{thm:rope_edge_distortion}
Assume semantic alignment is approximately uniform across rotary subspaces, i.e., $\rho_m\approx\rho_{\mathrm{total}}/N$ with $\rho_{\mathrm{total}}>0$. Then, for two token positions separated by $\delta$, the expected pre-softmax attention score is
\begin{equation}
\begin{aligned}
\mathbb{E}[s(i,j)]
&\approx
\frac{\rho_{\mathrm{total}}}{\sqrt{d}}\,r_b(\delta), \\
r_b(\delta)
&:=
\frac{1}{\ln b}
\int_{1/b}^{1}\frac{\cos(\delta\theta)}{\theta}\,d\theta
=
\frac{\mathrm{Ci}(\delta)-\mathrm{Ci}(\delta/b)}{\ln b},
\end{aligned}
\label{eq:expected_attention_log}
\end{equation}
where $r_b(\delta)$ is the normalized RoPE retention factor and $\mathrm{Ci}(\cdot)$ is the cosine integral. In the long-range regime $1\ll \delta \ll b$, this retention factor is approximated by
\begin{equation}
r_b(\delta)
\approx
\bar r_b(\delta)
:=
1-\frac{\ln\delta+\gamma}{\ln b},
\label{eq:retention_envelope}
\end{equation}
where $\gamma$ is the Euler--Mascheroni constant.
\end{theorem}

\begin{proof}
The proof is provided in \Cref{app:thm1_proof}.
\end{proof}

\noindent \textbf{Intuition.}
\Cref{thm:rope_edge_distortion} shows that RoPE converts token distance into a predictable reduction in attention alignment. The term $r_b(\delta)$ in \Cref{eq:expected_attention_log} is the normalized retention factor induced by RoPE, and \Cref{eq:retention_envelope} shows that it decreases approximately logarithmically with token distance. Thus, even when two graph-adjacent nodes are semantically aligned before positional encoding, placing them far apart in the serialized prompt can reduce their expected attention through RoPE-induced phase mismatch.
This logarithmic form follows directly from RoPE's log-spaced frequencies. For a gap $\delta$, subspaces with $\delta\theta_m \lesssim 1$ remain nearly position-insensitive as their cosine modulation stays close to one, while subspaces with $\delta\theta_m \gtrsim 1$ contribute more strongly to decay. Since RoPE frequencies are log-spaced, the fraction of such active subspaces grows approximately as $\ln\delta/\ln b$, yielding the decay rate in \Cref{eq:retention_envelope}.

\subsection{Extension to Graph Topology via Bandwidth}
\label{sec:bandwidth_bound}

The preceding result characterizes attention decay as a function of token distance $\delta$. We now connect this sequence-level effect to graph topology by substituting the worst-case token stretch forced by graph bandwidth into the same retention envelope $\bar r_b(\delta)$.

\begin{definition}[Graph Bandwidth]
For a graph $G=(V,E,X)$, its bandwidth is the smallest possible worst-case edge stretch over node orderings $\pi$, denoted as
\[
B(G):=\min_{\pi:V\to\{1,\ldots,|V|\}}\max_{(u,v)\in E}|\pi(u)-\pi(v)|.
\]
\label{def:graph_bandwidth}
\end{definition}

\begin{corollary}[Bandwidth-Limited Attention Retention]
\label{cor:bandwidth_retention}
Assume each node has approximately $\tau$ tokens. Let $\delta^L_{u,v}$ be the token-level edge stretch under linearization $L$, and define the best possible worst-edge smooth retention as
$\bar r_b^*(G):=\max_L\min_{(u,v)\in E}\bar r_b(\delta^L_{u,v})$.
In the regime $1\ll \tau B(G)\ll b$,
\begin{equation}
\bar r_b^*(G)
\lesssim
\bar r_b(\tau B(G))
=
1
-
\frac{\ln B(G)+\ln\tau+\gamma}{\ln b}.
\label{eq:bandwidth_retention_bound}
\end{equation}
\end{corollary}

\begin{proof}
The proof is provided in \Cref{app:bandwidth_bound_proof}.
\end{proof}

\noindent \textbf{Implication.}
Corollary 4.3 connects attention distortion and graph bandwidth. Expanding the logarithm in \Cref{eq:bandwidth_retention_bound} gives
\begin{equation}
\bar r_b^*(G)
\lesssim
1
-
\underbrace{\frac{\ln B(G)}{\ln b}}_{\text{Graph}}
\ -
\underbrace{\frac{\ln \tau}{\ln b}}_{\text{Text}}
\ -
\underbrace{\frac{\gamma}{\ln b}}_{\text{RoPE offset}}.
\label{eq:bandwidth_decomposition}
\end{equation}
The text-length term captures the cost of representing each node with natural language, while the graph term captures the additional cost imposed by the graph structure. As the loss is logarithmic, each doubling of either $B(G)$ or $\tau$ decreases the retention envelope by $\ln 2/\ln b$, which is approximately $7.5\%$ when $b=10^4$. Thus, even under the best possible ordering, high-bandwidth graph topologies induce unavoidable attention attenuation.
\subsubsection{Validating \Cref{thm:rope_edge_distortion}}
\label{sec:validating_rope_decay}

We analyze three graph families: Erdős--Rényi (ER), Stochastic Block Models (SBM), and Barabási--Albert (BA). For each, we generate 1,000 graphs with $n=30$ nodes and serialize them using Random, Breadth-First Search (BFS), and Reverse Cuthill--McKee (RCM) orderings; details are given in \Cref{sec:synth_graphs}. RCM is a bandwidth-reduction heuristic designed to reduce edge stretch during linearization. We then simulate the LLM by assigning queries and keys according to the theoretical setup, applying RoPE, and computing the resulting node-level attention similarity passed through a sigmoid (denoted as $\sigma$).

\vspace{0.1cm}
\noindent \textbf{Findings.}
As shown in \Cref{fig:attn_sim}, attention similarity consistently decays as $\delta_{u,v}$ increases, dropping by up to 50\% and qualitatively matching the theoretical prediction. Crucially, while bandwidth-reduction heuristics such as RCM reduce the frequency of large gaps, they cannot eliminate the residual distortion imposed by graph topology. When large gaps remain, RoPE still induces the same attention attenuation. This confirms that improved serialization is only a partial solution: the fundamental mismatch between graph topology and sequence-based positional geometry persists even under optimized orderings.
\section{Linking Attention Behavior to Performance}

While our theory suggests that RoPE attenuates attention over long distances, validating this behavior requires separating structural versus semantic reasoning. In previous studies, correct predictions have been shown to arise from superficial text cues rather than structure, making it difficult to disentangle the two \cite{chen2024exploringpotentiallargelanguage}. To parse these effects, we propose a semi-synthetic task based on real-world graphs that retains structural complexity, such as power-law degree distributions and community structure, while controlling for semantic signals. 
\begin{figure*}[t]
    \centering
    \includegraphics[width=.92\linewidth]{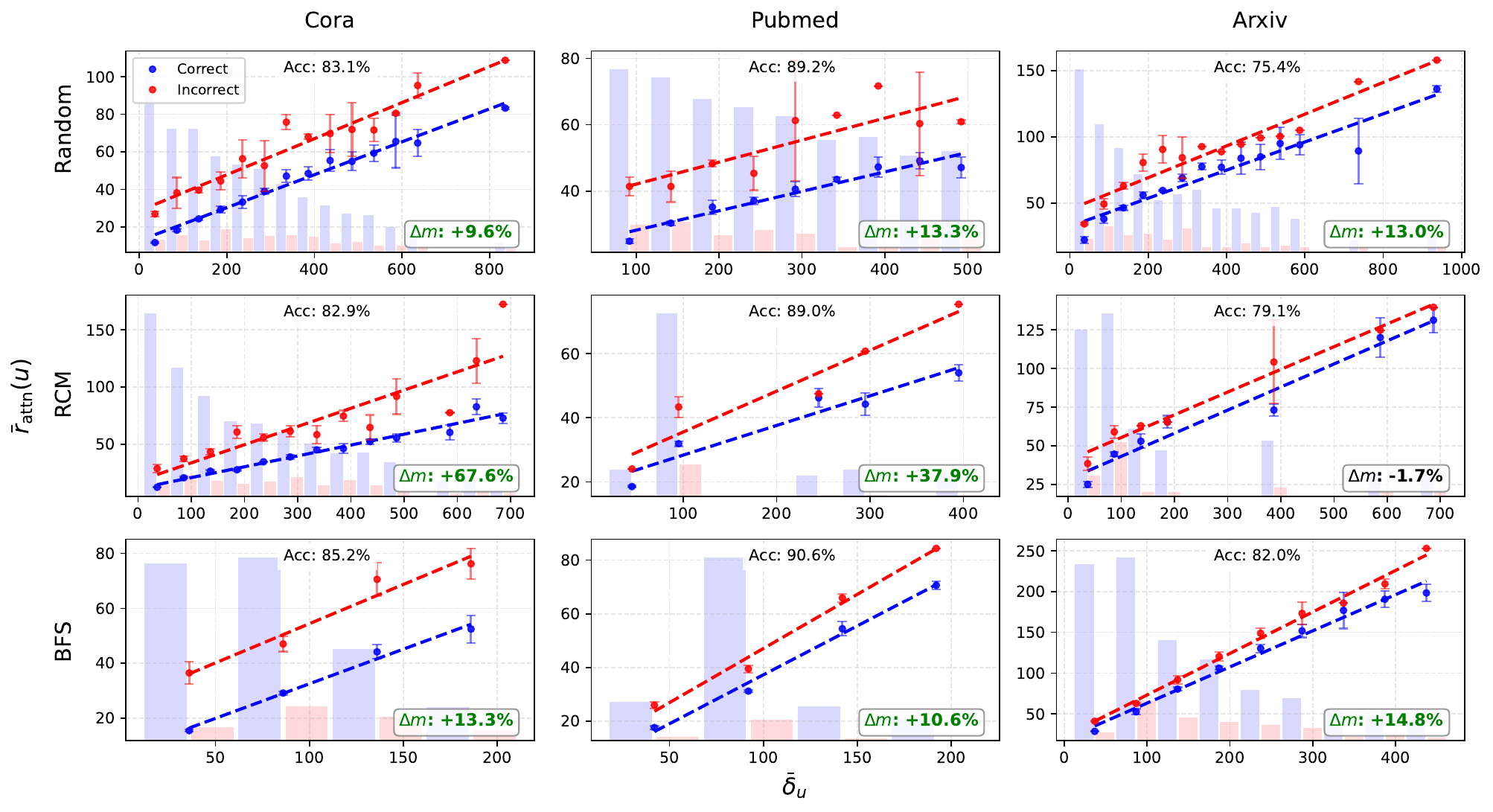}
    \vspace{-0.5cm}
    \Description{A wide figure displaying a 3-by-3 grid of scatter plots. Each subplot plots average neighbor attention rank on the vertical axis against average token distance on the horizontal axis. The data points are divided into two distinct groups: blue markers representing correct classifications and red markers representing incorrect ones. Each group has a corresponding dashed trendline fit to the data. Across all nine panels, the red markers are generally positioned higher on the vertical axis than the blue markers at similar horizontal positions, and the red dashed trendlines exhibit a visibly steeper upward slope compared to the blue dashed trendlines.}
    \caption{
        We plot the average neighbor attention rank $\bar{r}_{\text{attn}}(u)$ as a function of the average token distance $\bar{\delta}_u$. Blue and red markers correspond to correctly and incorrectly classified nodes, respectively, with dashed lines denoting linear fits computed over bins that contain both groups. Across all settings, incorrectly classified nodes exhibit consistently higher rank values at comparable token distances and a systematically steeper trendline, with the slope for incorrect nodes increasing by $\Delta m$\% relative to correct nodes. This percent increase in slope indicates that attention to graph-adjacent neighbors degrades more rapidly with token distance for incorrectly classified nodes, implicating diminished neighbor prioritization as a driver of prediction error.}
    \label{fig:attention_rank_analysis}
    \vspace{-0.2cm}
\end{figure*}
\begin{figure}[ht!]
    \centering
    \includegraphics[width=0.9\linewidth]{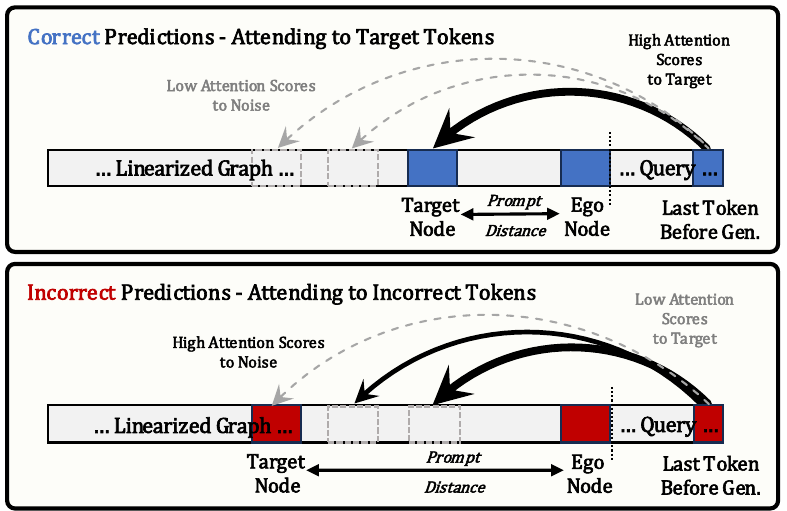}
    \Description{A schematic diagram illustrating a sequence of tokens as a long horizontal gray bar. Within the bar, specific segments are highlighted in different colors to denote the target node, ego node, and final token. A labeled bracket between the target and ego nodes indicates the token distance. Curved arrows arch backward from the final token to earlier segments in the sequence, representing attention. The diagram visually contrasts two scenarios: a 'correct' case where the thickest curved arrow lands directly on the target node, and an 'incorrect' case where the primary arrow diverts to a different, non-target segment.}
    \caption{Evaluation protocol. The gray bar indicates the prompt containing the graph and query. Colored segments mark the target node, ego node, and the final token. The token distance between the target and ego nodes defines $\delta_{u,v}$, while curved arrows indicate attention strength $s_{\text{attn}}(f, v)$. In correct predictions, attention peaks on the target node; in incorrect cases, attention is diverted to non-target nodes.}
    \label{fig:eval_proto}
    \vspace{-0.4cm}
\end{figure}

\subsection{Controlled Neighbor Aggregation Task}

\vspace{0.1cm}
\noindent \textbf{Task.}
We begin by extracting 2-hop subgraphs from graph benchmarks including Cora \cite{cora}, PubMed \cite{pubmed}, Arxiv \cite{arxiv} to preserve realistic graph properties. The text is provided by \cite{wang2025exploringgraphtaskspure}. Then, we generate LLM-as-Predictor-style prompts, where we specify the node, its label, and its neighbors. The goal for the LLM is to predict the \textit{majority class label} among a specific target node's neighbors. An example is seen below, and further details on prompt design are provided in \Cref{sec:cnat}.

\begin{tcolorbox}[colback=gray!10, colframe=gray!40, boxrule=0.5pt, arc=2pt, left=6pt, right=6pt, top=6pt, bottom=6pt]
\small
\textbf{Example Prompt (Derived from Cora):}\\
\texttt{Node 10 has label <Neural Networks> and neighbors \{5, 8\}.}\\
\texttt{Node 5 has label <Theory> and neighbors \{10\}.}\\
\texttt{Node 8 has label <Theory> and neighbors \{10\}.}\\
\texttt{Node 0 has label <Masked> and neighbors \{5, 8, 10\}.}\\
\textbf{Task:} \texttt{What is the most common label across Node 0's neighbors? Choose from: [Theory, Neural Networks, ...].}\\
\textbf{Answer:} \texttt{Theory}
\end{tcolorbox}

\vspace{0.1cm}
\noindent \textbf{Intuition}. This task probes an LLM’s ability to perform a fundamental operation in graph learning: neighborhood aggregation \citep{hamilton2018inductiverepresentationlearninglarge}. 
Moreover, we conceptualize the raw node text in TAGs as a noisy approximation of a node's label. Thus, in a realistic forward pass, an LLM must first de-noise this text to extract the relevant signal before aggregation. By directly exposing the neighboring labels, we simulate an idealized scenario where this extraction is perfect, isolating the \textit{structural aggregation} from the \textit{semantic extraction}. As a result, this task enforces a strict dependency on the graph where the answer can only be derived by retrieving and aggregating the neighbor signals. Consequently, incorrect predictions imply that the LLM failed to process the serialized graph topology.

\vspace{0.1cm}
\noindent \textbf{Setup.}
Our analysis combines performance with an examination of attention patterns to diagnose the LLM's behavior. For each node \(u\), we identify the neighbors \(\mathcal{N}(u)\) appearing in the prompt. For every \(v \in \mathcal{N}(u)\), let \(\mathcal{T}_v\) denote the token span corresponding to its label, and \(\mathcal{T}_u\) the token span for the query node. All experiments use Qwen2.5-3B-Instruct.
As Qwen2.5 is a decoder-only LLM, the final prediction is generated from the hidden state of the last token \(t_f\). Consequently, all information required to answer the query, most critically, the distribution of neighbor labels, must be routed to \(t_f\) through the attention mechanism. We use this to inform our evaluation protocol. 

\vspace{0.1cm}
\noindent \textbf{Evaluation.}
We characterize each neighbor \(v \in \mathcal{N}(u)\) using two quantities. First, \(\delta_{u,v}\) measures the separation between $u$ and $v$ in the prompt, which is derived from the separation in their respective first tokens (e.g. '\underline{Node} X...'). Second, while attention similarity captures raw magnitude, the actual influence of a token after softmax depends on its relative standing against other tokens in the sequence.  Therefore, rather than relying on magnitude alone, we assess whether a neighbor's tokens are prioritized or suppressed. A pictorial example of how these quantities are extracted is provided in \Cref{fig:eval_proto}.

Specifically, we define the \textit{token rank} of a token at position $j$ as the count of other tokens $k$ in the sequence that receive a higher score:

{\small
\begin{equation*}
r(t_{f}, j) \;=\; \sum_{k \in S, k \neq j} \mathbb{I}\!\left( s(t_{f}, k) > s(t_{f}, j) \right).
\end{equation*}
}
A rank of 0 indicates perfect prioritization. To characterize a target node $v$, we aggregate this metric by averaging over its tokens $\mathcal{T}(v)$:

\vspace{-0.3cm}
{\small
\begin{equation*}
r_{\text{attn}}(v) \;=\; \frac{1}{|\mathcal{T}(v)|} \sum_{j \in \mathcal{T}(v)} r(t_{f}, j).
\end{equation*}
}
To quantify distortion for a node $u$, we summarize the neighborhood impact by averaging both quantities across all neighbors $v \in \mathcal{N}(u)$:

\vspace{-0.3cm}
{\small
\begin{equation*}
\bar{\delta}_u \;=\; \frac{1}{|\mathcal{N}(u)|} \sum_{v \in \mathcal{N}(u)} \delta_{u,v}\, , \qquad \bar{r}_{\text{attn}}(u) \;=\; \frac{1}{|\mathcal{N}(u)|} \sum_{v \in \mathcal{N}(u)} r_{\text{attn}}(v).
\end{equation*}
}
Relating average rank $\bar{r}_{\text{attn}}(u)$ to average separation $\bar{\delta}_u$ mechanistically probes the model’s ability to distinguish relevant neighbors from those deprioritized solely due to linearization artifacts.

Notably, attention heads in LLMs are known to be functionally specialized, with only a subset contributing to reasoning, while others encode auxiliary or redundant patterns \citep{voita-etal-2019-analyzing}. Thus, aggregating all heads can obscure the effects induced by serialization. To focus on heads most relevant with the input structure, we analyze the top-\(k\) attention heads (\(k=25\%\)), selected by their \emph{pre-RoPE} attention scores.
Similarly, in instruction-tuned LLMs, the final layers often shift from reasoning to response formatting due to human feedback-style training \citep{miao2025energylossphenomenonrlhf}. Since our evaluation aims to probe how attention distinguishes relevant neighbors within the prompt, we analyze late, but not final, layers, where attention remains tied to reasoning (layer 29 for Qwen2.5-3B). Further details are provided in \Cref{sec:cnat}.

\vspace{0.1cm}
\noindent \textbf{Findings.}
In \Cref{fig:attention_rank_analysis}, we examine how attention rank $\bar{r}_{\text{attn}}(u)$ varies with token distance $\bar{\delta}_u$, stratified by prediction correctness. We restrict to bins containing both correct and incorrect outcomes, isolating attention patterns associated with success or failure rather than uniformly easy or hard regions, leading to two patterns. First, incorrect predictions consistently have higher attention ranks than correct predictions at comparable distances, indicating that failures occur when graph-relevant neighbor tokens are not prioritized over task-irrelevant tokens. Second, attention quality degrades with token distance: $\bar{r}_{\text{attn}}(u)$ increases as $\bar{\delta}_u$ grows, with a steeper slope for incorrect predictions. The large values of $\Delta m$ further show that failed predictions suffer a faster loss of attention with distance. Together, these results reveal an attention-level failure mode in serialized graph reasoning: when graph evidence lies beyond the model's effective attention range, its influence weakens, leading to incorrect predictions.

\begin{figure}
    \centering
    \includegraphics[width=1.0\linewidth]{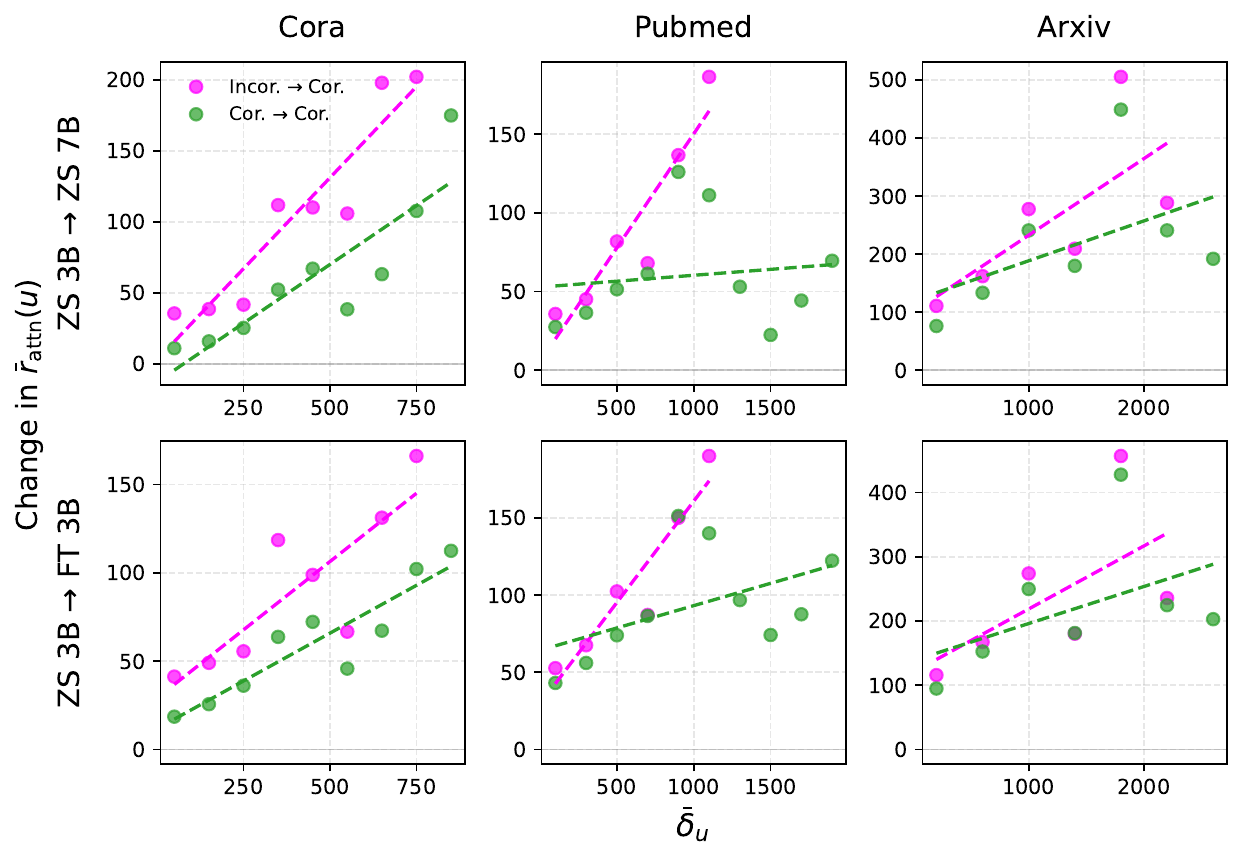}
    \vspace{-0.5cm}
    \Description{A two row figure with three columns. The top row corresponds to increasing model size, and the bottom row corresponds to fine-tuning. Both rows plot the change in average attention rank for datasets Cora, Pubmed, and Arxiv. The vertical axis denotes the change in attention rank, and the horizontal axis denotes the average token distance. The plotted data in each panel compares groups of nodes, specifically contrasting newly corrected nodes against those that remained correct. Visually, all groups display an upward trend as the horizontal value increases, but the data points or curves representing the 'Correct to Correct' group exhibit a noticeably flatter or smaller upward shift compared to the steeply rising trend of the newly corrected nodes.}
    \caption{We plot the change in $\bar{r}_{\text{attn}}(u)$ for nodes corrected by increasing model size (top) and fine-tuning (bottom), alongside nodes that stayed correct. Positive changes indicate improvements in attention rank, i.e., relevant neighbors moving to lower ranks. While both groups show improved attention ranking with increasing $\bar{\delta}_u$, the Correct $\to$ Correct cases exhibit smaller changes. This indicates that while mitigation strategies generally improve attention alignment, they induce larger rank shifts for nodes whose predictions are corrected.}
\label{fig:mitigation_rank_improvement}
\vspace{-0.3cm}
\end{figure}

\subsection{Effectiveness of Current Mitigation Strategies}

We now probe how standard mitigation strategies like fine-tuning (FT) and model scaling and see how they change attention modulation. Specifically, we repeat the experiments in the previous section but now (a) fine-tune the Qwen2.5-3B on the aggregation task and (b) perform zero-shot evaluation on a larger Qwen2.5-7B model.

\vspace{0.1cm}
\noindent\textbf{Findings.} Our analysis indicates that both mitigation strategies improve performance by alleviating the attention degradation identified earlier. As shown in \Cref{fig:mitigation_rank_improvement}, positive changes in $\bar{r}_{\text{attn}}(u)$ occur consistently across nodes, where \textit{change is computed as the pre-mitigation rank minus the post-mitigation rank}. Thus, positive values indicate that relevant neighbors move to better, lower attention ranks, though with different magnitudes. Nodes that transition from {incorrect $\to$ correct} exhibit larger attention-rank improvements than those that remain {correct}, suggesting that while these methods broadly refine attention, correcting errors requires a more pronounced alignment between attention and graph structure.
Moreover, this effect is distance-dependent, indicating that fine-tuning and model scaling actively correct the distance-induced suppression that affects far neighbors in the serialized graph.
Taken together, these findings suggest that performance gains from these strategies arise, in part, from improved alignment between attention and graph structure. This observation motivates the question of whether benefits can be achieved more directly, without the computational overhead of larger models or the data requirements of task-specific fine-tuning.

\section{GaLA: Graph-aligned Language Attention}
\label{sec:gala}

Building on our analysis showing that linearization-induced attention decay is a contributor to poor zero-shot performance, we propose a direct intervention at the level of attention itself. We introduce Graph-Aligned Language Attention (GaLA), a lightweight method that applies a graph-aware attention modification at inference time to a frozen LLM backbone. At a high level, GaLA softly encourages stronger attention between tokens corresponding to graph-adjacent nodes, while preserving the model’s pre-trained linguistic behavior.

\subsection{Structure-Aware Attention Injection}

For a graph $G$, let $\mathcal{M}:\{1,\dots,T\}\rightarrow V\cup\{\varnothing\}$ map token indices to $G$'s nodes (when applicable). For layer $l$ and attention head $h$, GaLA augments the pre-softmax attention logits with a graph-aware bias. Specifically, the modified attention logit between tokens $i$ and $j$ is

{\small
\begin{equation*}
\tilde{a}^{(l,h)}_{i,j}
=
\underbrace{
\frac{
(\tilde{\mathbf{q}}^{(l,h)}_i)^\top
\tilde{\mathbf{k}}^{(l,h)}_j
}{\sqrt{d}}
}_{\text{semantic attention}}
\;+\;
\underbrace{
\lambda_h \, B_{i,j}
}_{\text{graph awareness}},
\end{equation*}
}where $\tilde{\mathbf{q}}^{(l,h)}_i = \mathbf{R}_\Theta(i)\mathbf{q}^{(l,h)}_i$ and $\tilde{\mathbf{k}}^{(l,h)}_j = \mathbf{R}_\Theta(j)\mathbf{k}^{(l,h)}_j$ are the RoPE-rotated queries and keys, and $d$ is the head dimension.
The structural bias $B_{i,j}$ is defined for a graph $G$ as:
{\small
\begin{equation*}
B_{i,j}
=
\begin{cases}
\dfrac{1}{d_G\!\left(\mathcal{M}(i),\,\mathcal{M}(j)\right)}
& \text{if } 
\substack{
\mathcal{M}(i),\mathcal{M}(j)\in V,\\
\mathcal{M}(i)\ne \mathcal{M}(j),\\
d_G\!\left(\mathcal{M}(i),\mathcal{M}(j)\right) < \infty
},\\[2ex]
0 & \text{otherwise}.
\end{cases}
\end{equation*}
}
where $d_G(u,v)$ denotes the shortest path distance between nodes $u$ and $v$. This formulation softly promotes attention between tokens of nearby graph nodes while leaving unrelated token pairs unchanged. The resulting attention weights are obtained via
{\small $
A^{(l,h)}_{i,j} = \mathrm{softmax}_j\!(\tilde{a}^{(l,h)}_{i,j}).
$}
$\lambda_h$ controls how each head incorporates $B_{i,j}$. When $\lambda_h \to 0$, the head behaves identically to the original LLM; when $\lambda_h > 0$, the head is biased toward the graph, enabling specialization without overwriting the LLM’s capabilities. We apply this injection only in the \emph{early layers} ($l \in [0,L/2]$, where $L$ is the number of layers), allowing graph structure to shape contextual representations while leaving later layers to translate these representations into natural language outputs. We next describe how $\lambda_h$ is computed.

\subsection{Deriving the Per-Head Modulation}
\label{sec:head_selection}

GaLA sets the modulation strength $\lambda_h$ for each head with an inexpensive, one-time calibration on a small set $\mathcal{D}_{\mathrm{cal}}$, performed before inference while leaving all model weights frozen. We provide two variants that differ only in whether they use calibration labels: a \emph{label-free} entropy heuristic, and a \emph{label-based} gradient score that identifies useful heads more precisely when labels are available.

\vspace{0.05cm}
\noindent\textbf{Entropy.}
Using no labels, GaLA sets $\lambda_h$ from each head's attention entropy, a lightweight proxy for its functional role:
{\small
\begin{equation*}
\lambda_h
=
\beta
\sqrt{
\frac{H(h)}
{\max_{h' \in \mathcal{H}_l} H(h')}
},
\qquad
H(h)
=
\mathbb{E}_{S \in \mathcal{D}_{\mathrm{cal}}}
\left[
-\frac{1}{T}
\sum_{i=1}^{T}
\sum_{j=1}^{i}
A^{(h)}_{i,j}
\log A^{(h)}_{i,j}
\right].
\end{equation*}
}
Here, the inner sum is restricted to $j \le i$ to respect the causal mask. Low-entropy heads concentrate on a few tokens, typically serving syntactic or positional roles, whereas high-entropy heads spread attention broadly to integrate context \cite{clark2019doesbertlookat}. The square root compresses the range so that, after per-layer normalization, $\lambda_h$ emphasizes moderate-to-high-entropy heads while leaving sharply focused, low-entropy heads largely untouched.

\vspace{0.05cm}
\noindent\textbf{Gradient.}
When calibration labels are available, e.g., the same small set used to tune $\beta$ or other hyperparameters, we can identify beneficial heads directly. We attach a per-head scale $g_h$ initialized to $1$ to each head's structural bias and freeze all model parameters, so that only $g_h$ can update. A single backward pass gives a head score based on the gradient of the ground-truth-label negative log-likelihood:
{\small
\begin{equation*}
\lambda_h
=
\beta
\frac{\max\{0,-G(h)\}}
{\max_{h' \in \mathcal{H}_l} \max\{0,-G(h')\}},
\quad
G(h)
=
\mathbb{E}_{S \in \mathcal{D}_{\mathrm{cal}}}
\left[
\frac{\partial}{\partial g_h}
\big(-\log p(y \mid S)\big)
\right].
\end{equation*}
}
A negative gradient indicates that increasing a head's structural bias lowers the loss, marking a head that benefits from graph structure. $\max\{0,-G(h)\}$ retains only the beneficial direction. This requires only a single backward pass for head selection and does not update any model weights; the LLM remains frozen. In both variants, $\mathcal{H}_l$ denotes the heads in layer $l$ and $\beta$ scales the overall graph influence.

\subsection{Relation to Graph Transformers}
\label{sec:relation_gt}

GaLA builds on an insight from graph transformers (GT) where self-attention can be made graph-aware by injecting connectivity directly into the attention mechanism \cite{dwivedi2021generalizationtransformernetworksgraphs, rampasek2023recipegeneralpowerfulscalable, ying2021transformersreallyperformbad}. Specifically, GaLA adopts an attention modulation strategy which encodes distances in the graph. 
While this has similarities to Graphormer's design \cite{ying2021transformersreallyperformbad}, GaLA departs from Graphormer in two ways to accommodate the LLM:

\begin{itemize}[leftmargin=*]
    \item \textbf{Token vs.\ Node:} GTs fundamentally operate on fixed node embeddings. In contrast, LLMs consume graphs serialized into raw text. GaLA bridges this by applying graph structure at the token level, mapping node distances onto token-to-token attention.
    
    \item \textbf{Frozen-Model:} Most GTs require training to attain structure-awareness. GaLA instead introduces structure through an attention bias applied to a frozen LLM, with optional one-time calibration to set its head-wise strengths. This enables off-the-shelf LLMs to perform graph reasoning without retraining their parameters.
\end{itemize}

\section{Empirical Analysis of GaLA}
\label{sec:empirical}

\begin{table*}[t]
    \centering
    \renewcommand{\arraystretch}{0.95}
    \setlength{\tabcolsep}{3.0pt}
    \caption{Evaluation on semi-synthetic and real-world datasets. We highlight whether methods possess \textbf{language capabilities}, incorporate \textbf{graph bias}, and operate \textbf{at inference time}. GaLA delivers gains across both tasks while applying only a lightweight inference-time modification with light calibration. GraphICL is omitted from Semi-Synth given there are no in-context examples.}
    \label{tab:main_results}
    \vspace{-0.3cm}
    \footnotesize
    \begin{tabular}{c l l ccc ccc ccc}
    \toprule
    \multicolumn{6}{c}{}
    & \multicolumn{3}{c}{\textbf{Semi-Synth. Perf.}}
    & \multicolumn{3}{c}{\textbf{Real-World Perf.}} \\
    \cmidrule(lr){7-9} \cmidrule(lr){10-12}

    \textbf{}
    & \textbf{Type}
    & \textbf{Configuration}
    & \makecell{\textbf{Language}\\\textbf{Capability}}
    & \makecell{\textbf{Graph}\\\textbf{Bias}}
    & \makecell{\textbf{Frozen}\\\textbf{LLM Backbone}}
    & \textbf{Cora}
    & \textbf{PubMed}
    & \textbf{Arxiv}
    & \textbf{Cora}
    & \textbf{PubMed}
    & \textbf{Arxiv} \\
    \midrule \midrule

    % =============================================
    % MODEL 1: Qwen2.5-3B
    % =============================================
    \multirow{8}{*}{\rotatebox{90}{\textbf{\textsc{Qwen2.5-3B}}}}
    & \multirow{2}{*}{\textit{Baselines}}
        & Random Lin.
        & \cmark & \xmark & \cmark
        & $82.90 \pm 0.10$ & $88.50 \pm 0.30$ & $74.48 \pm 0.83$
        & $64.90 \pm 2.90$ & $77.80 \pm 0.60$ & $33.80 \pm 1.20$ \\
    &   & BFS Lin.
        & \cmark & \xmark & \cmark
        & $83.30 \pm 1.50$ & $90.30 \pm 0.70$ & $80.28 \pm 1.05$
        & $65.90 \pm 1.90$ & $77.50 \pm 0.90$ & $35.60 \pm 1.06$ \\
    \cmidrule(lr){2-12}

    & \multirow{4}{*}{\textit{Interventions}}
    & Chain-of-Thought (256)
        & \cmark & \xmark & \cmark
        & $93.57 \pm 0.55$ & $82.20 \pm 0.61$ & $78.50 \pm 0.62$
        & $32.80 \pm 0.71$ & $59.55 \pm 0.64$ & $28.15 \pm 0.07$ \\
    &  & Chain-of-Thought (512)
        & \cmark & \xmark & \cmark
        & $94.21 \pm 0.42$ & $88.78 \pm 0.43$ & $86.72 \pm 0.39$
        & $62.25 \pm 0.21$ & $84.10 \pm 0.71$ & $39.05 \pm 0.64$ \\
    &  & GraphICL
        & \cmark & \xmark & \cmark
        & \textcolor{gray}{N/A} & \textcolor{gray}{N/A} & \textcolor{gray}{N/A}
        & $53.90 \pm 0.70$ & $74.50 \pm 0.30$ & $32.10 \pm 1.30$ \\
    &   & Fine-Tuning
        & \cmark & \xmark & \xmark
        & $98.80 \pm 0.53$ & $98.90 \pm 0.26$ & $97.39 \pm 0.48$
        & $79.81 \pm 0.99$ & $89.80 \pm 1.13$ & $68.15 \pm 3.89$ \\
    \cmidrule(lr){2-12}

    & \multirow{2}{*}{\textbf{Ours}}
        & \textbf{GaLA} (Entropy)
        & \cmark & \cmark & \cmark
        & $89.10 \pm 1.50$ & $91.50 \pm 0.10$ & $83.06 \pm 0.89$
        & $67.00 \pm 2.20$ & $79.20 \pm 0.80$ & $35.80 \pm 0.99$ \\
    &   & \textbf{GaLA} (Gradient)
        & \cmark & \cmark & \cmark
        & $90.40 \pm 1.40$ & $92.00 \pm 1.00$ & $82.37 \pm 0.89$
        & $67.40 \pm 2.20$ & $82.50 \pm 2.10$ & $36.80 \pm 0.99$ \\
    \midrule \midrule

    % =============================================
    % MODEL 2: Ministral-3B
    % =============================================
    \multirow{8}{*}{\rotatebox{90}{\textbf{\textsc{Ministral-3B}}}}
    & \multirow{2}{*}{\textit{Baselines}}
        & Random Lin.
        & \cmark & \xmark & \cmark
        & $71.30 \pm 0.70$ & $82.00 \pm 0.20$ & $83.20 \pm 0.88$
        & $55.60 \pm 1.20$ & $73.40 \pm 0.00$ & $20.20 \pm 0.28$ \\
    &   & BFS Lin.
        & \cmark & \xmark & \cmark
        & $69.30 \pm 1.70$ & $80.00 \pm 0.60$ & $84.40 \pm 0.35$
        & $56.20 \pm 2.60$ & $74.40 \pm 0.80$ & $19.60 \pm 0.64$ \\
    \cmidrule(lr){2-12}

    & \multirow{4}{*}{\textit{Interventions}}
    & Chain-of-Thought (256)
        & \cmark & \xmark & \cmark
        & $81.43 \pm 2.06$ & $80.10 \pm 1.55$ & $64.31 \pm 0.55$
        & $9.95 \pm 0.49$ & $13.65 \pm 1.48$ & $0.95 \pm 0.21$ \\
    &  & Chain-of-Thought (512)
        & \cmark & \xmark & \cmark
        & $91.17 \pm 0.45$ & $93.87 \pm 0.86$ & $91.52 \pm 0.43$
        & $35.55 \pm 0.92$ & $53.55 \pm 0.07$ & $9.45 \pm 0.07$ \\
    &  & GraphICL
        & \cmark & \xmark & \cmark
        & \textcolor{gray}{N/A} & \textcolor{gray}{N/A} & \textcolor{gray}{N/A}
        & $56.40 \pm 0.40$ & $73.60 \pm 1.00$ & $15.60 \pm 0.20$ \\
    &   & Fine-Tuning
        & \cmark & \xmark & \xmark
        & $99.60 \pm 0.35$ & $99.43 \pm 0.23$ & $98.34 \pm 0.52$
        & $86.33 \pm 1.34$ & $92.50 \pm 0.42$ & $67.00 \pm 1.56$ \\
    \cmidrule(lr){2-12}

    & \multirow{2}{*}{\textbf{Ours}}
        & \textbf{GaLA} (Entropy)
        & \cmark & \cmark & \cmark
        & $79.50 \pm 0.50$ & $83.10 \pm 0.50$ & $86.20 \pm 0.68$
        & $56.10 \pm 2.30$ & $74.30 \pm 0.10$ & $18.20 \pm 0.84$ \\
    &   & \textbf{GaLA} (Gradient)
        & \cmark & \cmark & \cmark
        & $87.90 \pm 2.30$ & $87.50 \pm 0.10$ & $85.60 \pm 0.68$
        & $56.70 \pm 2.10$ & $74.10 \pm 0.90$ & $17.80 \pm 0.84$ \\
    \midrule \midrule

    % =============================================
    % MODEL 3: Qwen2.5-7B
    % =============================================
    \multirow{8}{*}{\rotatebox{90}{\textbf{\textsc{Qwen2.5-7B}}}}
    & \multirow{2}{*}{\textit{Baselines}}
        & Random Lin.
        & \cmark & \xmark & \cmark
        & $92.40 \pm 0.60$ & $92.80 \pm 0.20$ & $87.24 \pm 0.41$
        & $67.50 \pm 1.50$ & $86.20 \pm 0.00$ & $46.40 \pm 0.43$ \\
    &   & BFS Lin.
        & \cmark & \xmark & \cmark
        & $92.40 \pm 0.20$ & $93.20 \pm 0.20$ & $90.11 \pm 0.61$
        & $68.20 \pm 3.20$ & $86.30 \pm 0.70$ & $46.80 \pm 0.28$ \\
    \cmidrule(lr){2-12}

    & \multirow{4}{*}{\textit{Interventions}}
    & Chain-of-Thought (256)
        & \cmark & \xmark & \cmark
        & $98.90 \pm 0.56$ & $96.10 \pm 0.79$ & $93.07 \pm 0.12$
        & $54.30 \pm 1.03$ & $68.95 \pm 0.98$ & $38.91 \pm 1.47$ \\
    &  & Chain-of-Thought (512)
        & \cmark & \xmark & \cmark
        & $99.12 \pm 0.41$ & $97.34 \pm 0.34$ & $93.81 \pm 0.24$
        & $67.00 \pm 0.85$ & $88.60 \pm 0.71$ & $51.70 \pm 2.40$ \\
    &  & GraphICL
        & \cmark & \xmark & \cmark
        & \textcolor{gray}{N/A} & \textcolor{gray}{N/A} & \textcolor{gray}{N/A}
        & $61.10 \pm 1.50$ & $84.40 \pm 1.20$ & $48.40 \pm 1.40$ \\
    &   & Fine-Tuning (LoRA)
        & \cmark & \xmark & \xmark
        & $96.57 \pm 0.57$ & $96.83 \pm 0.90$ & $93.18 \pm 0.76$
        & $70.85 \pm 0.49$ & $90.05 \pm 0.07$ & $63.15 \pm 0.21$ \\
    \cmidrule(lr){2-12}

    & \multirow{2}{*}{\textbf{Ours}}
        & \textbf{GaLA} (Entropy)
        & \cmark & \cmark & \cmark
        & $93.70 \pm 0.30$ & $93.80 \pm 0.00$ & $90.95 \pm 0.59$
        & $68.60 \pm 2.80$ & $88.10 \pm 1.10$ & $45.40 \pm 0.28$ \\
    &   & \textbf{GaLA} (Gradient)
        & \cmark & \cmark & \cmark
        & $93.20 \pm 0.60$ & $94.20 \pm 0.20$ & $91.18 \pm 0.59$
        & $69.00 \pm 3.20$ & $87.70 \pm 0.10$ & $47.20 \pm 0.28$ \\
    \bottomrule
    \end{tabular}
    \end{table*}

In this section, we evaluate GaLA across two regimes. First, we revisit the semi-synthetic task to verify that GaLA can improve this task without training. Then, we apply GaLA to real-world node classification, demonstrating that restoring attention alignment can provide accuracy gains on standard TAG benchmarks.

\vspace{0.1cm}
\noindent \textbf{Setup.}
We evaluate GaLA across three LLMs: \textbf{Qwen2.5-3B-Instruct}, \textbf{Ministral-3B-Instruct}, and \textbf{Qwen2.5-7B-Instruct}, selected to assess model scale and different architectures. Unless otherwise stated (e.g., for the random linearization baseline), all methods utilize {BFS linearization} to minimize edge stretch. We compare our method against standard zero-shot inference, \textbf{Chain-of-Thought (CoT)} prompting with different generation lengths (256/512) to test reasoning-based mitigation, \textbf{GraphICL} \cite{sun2025graphiclunlockinggraphlearning} (real-world only), an in-context-learning baseline that places labeled neighbor and demonstration examples in the prompt, and \textbf{Fine-Tuning} as a performance upper bound. For Qwen2.5-7B-Instruct, we use LoRA, while for the others we perform full Fine-Tuning. GaLA biases attention only in the early layers ($[0, L/2]$) and assigns continuous bias weights across heads via either a label-free entropy heuristic or a labeled-calibration gradient score; in both cases the strength $\beta$ is tuned on a small validation set of 200 nodes. More details are provided in \Cref{sec:cnat}.

\subsection{Semi-Synthetic Aggregation and GaLA}

We begin with the semi-synthetic task. As shown in the left half of \Cref{tab:main_results}, GaLA consistently improves over standard linearization. Relative to BFS, GaLA yields gains of up to \textbf{18.6\%} on Ministral-3B ($69.3\% \rightarrow 87.9\%$) and \textbf{7.1\%} on Qwen2.5-3B ($83.3\% \rightarrow 90.4\%$) for Cora, far exceeding the gains from switching Random to BFS alone. This suggests that \textit{serialization order alone is insufficient} and correcting attention alignment provides an additional benefit.
The same trend holds on PubMed and Arxiv. On PubMed, GaLA improves Qwen2.5-3B from $90.3\%$ to $92.0\%$ and Ministral-3B from $80.0\%$ to $87.5\%$, recovering a meaningful fraction of the gap to fine-tuned models. While CoT can match or exceed GaLA in some cases, it requires substantially higher inference cost (see \Cref{fig:runtime}); moreover, GaLA outperforms CoT-512 on PubMed for Qwen2.5-3B ($92.0\%$ vs.\ $88.78\%$), indicating that explicit reasoning chains do not always compensate for structural distortion. Finally, gains are largest on smaller backbones and more modest on Qwen2.5-7B, suggesting that larger models are more robust to linearization and leave less distortion for GaLA to correct.

\subsubsection{Runtime Costs} 
Figure~\ref{fig:runtime} plots accuracy versus total runtime to process the test nodes for Qwen2.5-3B across Cora, PubMed, and Arxiv for the semi-synthetic task.  Focusing on the inference-time methods, across all datasets, GaLA operates at a similar runtime as BFS, while CoT incurs orders-of-magnitude overhead.
On Cora, BFS completes in $225$ seconds and GaLA in $\sim$$265$ seconds (including its one-time calibration), whereas CoT with 256 and 512 tokens require over $5000$ seconds, more than a $20\times$ slowdown. 
Overall, GaLA achieves accuracy gains at a small overhead cost, especially when compared to CoT, making it well-suited for practical use.

\begin{figure}
    \centering
    \includegraphics[width=1.00\linewidth]{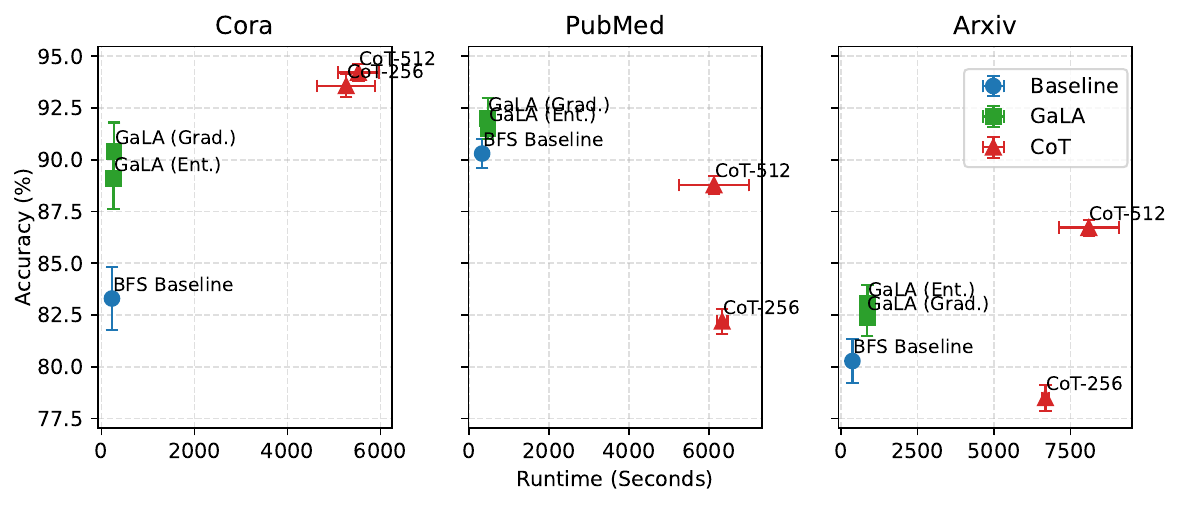}
        \Description{A plot graphing accuracy on the vertical axis against runtime on the horizontal axis. Data points representing various inference-time methods are plotted across the chart. The points for the GaLA method are positioned higher on the vertical axis but remain far to the left on the horizontal axis, aligning closely with the baseline method's low runtime. In contrast, the points for the CoT method are located far to the right, indicating a significantly higher runtime. An arrow on the chart highlights the upward shift toward higher accuracy without a corresponding horizontal shift in runtime.}
    \caption{
    Accuracy versus runtime for Qwen2.5-3B across datasets for inference-time methods.GaLA delivers accuracy gains at near-baseline cost, while CoT incurs large overhead.}
    \label{fig:runtime}
\end{figure}

\subsection{Real-World Node Classification}

Next, we evaluate whether GaLA transfers to real-world tasks. As shown in the right half of \Cref{tab:main_results}, GaLA improves performance in most settings, with gains often reaching several percentage points and peaking at 5.0\% on PubMed for Qwen2.5-3B (77.5\% to 82.5\%). This suggests that structural alignment remains beneficial even when models have access to rich node text, rather than only in the controlled semi-synthetic setting. Notably, GaLA also surpasses {GraphICL} by $4.7$--$13.5$\% on Qwen2.5-3B and in nearly all settings across the other backbones. These findings demonstrate that directly aligning attention with graph structure provides a more effective, and more label-efficient, use of structure than supplying it as in-context examples.
That said, the narrower margin compared to the semi-synthetic task is expected, as real-world TAGs provide strong semantic signals that can partially, and sometimes fully as argued in \cite{chen2024exploringpotentiallargelanguage} for PubMed, compensate for structural distortion. Consequently, while serialization-induced attention degradation is less dominant, it is not entirely resolved by text alone. The continued gains from GaLA suggest that positional misalignment persists even when text-based cues are available, motivating future work on when structural alignment is most useful across domains with varying textual and relational signal.

\section{Conclusion}

In this work, we moved beyond characterizing \textit{when} LLMs fail at graph reasoning to diagnosing a key mechanism behind \textit{why}, identifying a geometric misalignment between graph topology and RoPE. We showed that when graph-adjacent nodes are separated in the serialized sequence, RoPE can systematically suppress their interaction, creating an attention bottleneck that degrades performance.
To address this issue, we introduced GaLA, a lightweight attention intervention deployed at inference time that biases attention toward the graph without retraining or architectural modification. By enabling graph awareness, GaLA tends to improve node classification performance across both semi-synthetic and real-world datasets.
Together, our findings highlight the importance of aligning positional inductive biases with underlying graph structure, and suggest that principled interventions can offer a scalable strategy to improve graph reasoning with LLMs.

\section{Acknowledgements}
This material is partially based upon work supported by National Science Foundation under Grants No. IIS~2212143 and IIS~2504090. We thank the anonymous reviewers for their valuable feedback.

\bibliographystyle{ACM-Reference-Format}
\balance
\bibliography{references}

\appendix

\section{Theoretical Analysis}
\label{app:theoretical_analysis}

\subsection{Derivation of \Cref{thm:rope_edge_distortion}}
\label{app:thm1_proof}

Let two tokens occur at positions $i$ and $j$, with separation
$\delta=|i-j|$. The node-level score $s_{\mathrm{attn}}(u,v)$ is the
average of this token-level quantity over all pairs in
$\mathcal{T}(u)\times\mathcal{T}(v)$.
We analyze a single attention head of dimension $d=2N$, consisting of
$N$ two-dimensional RoPE subspaces. In subspace $m$, the RoPE frequency is
$\theta_m=b^{-m/N}$, and the pre-RoPE query and key components are
$\mathbf{q}^{(m)}_i,\mathbf{k}^{(m)}_j\in\mathbb{R}^2$. Let
\[
\boldsymbol{\Sigma}_m
=
\mathbb{E}\!\left[
\mathbf{k}^{(m)}_j(\mathbf{q}^{(m)}_i)^\top
\right]\in\mathbb{R}^{2\times2}.
\]
Its trace gives the pre-RoPE semantic alignment,
\[
\rho_m
=
\mathbb{E}\!\left[
(\mathbf{q}^{(m)}_i)^\top \mathbf{k}^{(m)}_j
\right]
=
\operatorname{Tr}\boldsymbol{\Sigma}_m,
\]
and its antisymmetric component is summarized by
\[
c_m
=
(\boldsymbol{\Sigma}_m)_{21}-(\boldsymbol{\Sigma}_m)_{12}.
\]

Let $\mathbf{R}_{p,\theta_m}$ denote the two-dimensional RoPE rotation at
position $p$ and frequency $\theta_m$. Then
\[
\tilde{\mathbf{q}}^{(m)}_i
=
\mathbf{R}_{i,\theta_m}\mathbf{q}^{(m)}_i,
\qquad
\tilde{\mathbf{k}}^{(m)}_j
=
\mathbf{R}_{j,\theta_m}\mathbf{k}^{(m)}_j.
\]
Thus the subspace contribution depends on
\begin{equation}
(\tilde{\mathbf{q}}^{(m)}_i)^\top
\tilde{\mathbf{k}}^{(m)}_j
=
(\mathbf{q}^{(m)}_i)^\top
\mathbf{R}_{j-i,\theta_m}
\mathbf{k}^{(m)}_j.
\label{eq:app_rope_dot}
\end{equation}
Taking expectations and using
\[
\mathbb{E}[\mathbf{x}^\top\mathbf{A}\mathbf{y}]
=
\operatorname{Tr}\!\left(
\mathbf{A}\,\mathbb{E}[\mathbf{y}\mathbf{x}^\top]
\right),
\]
we obtain
\[
\mathbb{E}\!\left[
(\tilde{\mathbf{q}}^{(m)}_i)^\top
\tilde{\mathbf{k}}^{(m)}_j
\right]
=
\operatorname{Tr}\!\left(
\mathbf{R}_{-\delta,\theta_m}\boldsymbol{\Sigma}_m
\right).
\]
For any $2\times2$ matrix $\boldsymbol{\Sigma}$,
\[
\operatorname{Tr}(\mathbf{R}_{-\varphi}\boldsymbol{\Sigma})
=
(\operatorname{Tr}\boldsymbol{\Sigma})\cos\varphi
+
(\boldsymbol{\Sigma}_{21}-\boldsymbol{\Sigma}_{12})\sin\varphi.
\]
Therefore
\begin{equation}
\mathbb{E}\!\left[
(\tilde{\mathbf{q}}^{(m)}_i)^\top
\tilde{\mathbf{k}}^{(m)}_j
\right]
=
\rho_m\cos(\delta\theta_m)
+
c_m\sin(\delta\theta_m).
\label{eq:app_subspace_cos}
\end{equation}
The first term is the contribution of pre-RoPE alignment, while the second
comes from the antisymmetric component. In the continuum
aggregation below, the skew contribution is proportional to
\[
\frac{\mathrm{Si}(\delta)-\mathrm{Si}(\delta/b)}{\ln b},
\]
which is asymptotically constant for $1\ll \delta\ll b$. Hence the skew term does not change the leading logarithmic
decay rate. Aggregating over subspaces and focusing on the distance-dependent alignment
term gives
\begin{equation}
\mathbb{E}[s(i,j)]
=
\frac{1}{\sqrt{d}}
\sum_{m=0}^{N-1}
\rho_m\cos(\delta\theta_m).
\label{eq:app_expected_score_discrete}
\end{equation}
Letting $\rho_{\mathrm{total}}=\sum_{m=0}^{N-1}\rho_m$, write
\begin{equation}
\mathbb{E}[s(i,j)]
=
\frac{\rho_{\mathrm{total}}}{\sqrt{d}}
r_{\mathrm{disc}}(\delta),
\qquad
r_{\mathrm{disc}}(\delta)
=
\frac{
\sum_{m=0}^{N-1}
\rho_m\cos(\delta\theta_m)
}{
\sum_{m=0}^{N-1}
\rho_m
}.
\label{eq:app_retention_disc}
\end{equation}
To obtain \Cref{thm:rope_edge_distortion},
assume $\rho_m\approx\rho_{\mathrm{total}}/N$. Then
\begin{equation}
r_{\mathrm{disc}}(\delta)
\approx
\frac{1}{N}
\sum_{m=0}^{N-1}
\cos(\delta b^{-m/N}).
\label{eq:app_uniform_rho_sum}
\end{equation}
As $N\to\infty$, set $x=m/N$, so $\theta=b^{-x}$ and
$dx=-d\theta/(\theta\ln b)$. Hence
\begin{equation}
\begin{aligned}
r_{\mathrm{disc}}(\delta)
&\to
\int_0^1 \cos(\delta b^{-x})\,dx 
=
\frac{1}{\ln b}
\int_{1/b}^{1}
\frac{\cos(\delta\theta)}{\theta}\,d\theta.
\end{aligned}
\label{eq:app_log_density_integral}
\end{equation}
Thus RoPE induces the log-uniform frequency density
\[
g_{\mathrm{RoPE}}(\theta)=\frac{1}{\theta\ln b},
\qquad \theta\in[1/b,1].
\]
Using the substitution $u=\delta\theta$,
\begin{equation}
r_b(\delta)
=
\frac{1}{\ln b}
\int_{\delta/b}^{\delta}
\frac{\cos u}{u}\,du
=
\frac{\mathrm{Ci}(\delta)-\mathrm{Ci}(\delta/b)}{\ln b},
\label{eq:app_exact_ci}
\end{equation}
where $d\,\mathrm{Ci}(u)/du=\cos u/u$. Combining
\Cref{eq:app_retention_disc} and \Cref{eq:app_exact_ci} gives
\[
\mathbb{E}[s(i,j)]
\approx
\frac{\rho_{\mathrm{total}}}{\sqrt{d}}
\frac{\mathrm{Ci}(\delta)-\mathrm{Ci}(\delta/b)}{\ln b}.
\]
For $1\ll \delta\ll b$, the lower endpoint $\delta/b$ is small, so
\[
\mathrm{Ci}(\delta/b)
\approx
\gamma+\ln(\delta/b).
\]
The omitted terms vanish as $\delta/b$ becomes small. Substituting this into
\Cref{eq:app_exact_ci} gives
\[
r_b(\delta)
\approx
1-\frac{\ln\delta+\gamma}{\ln b}
+
\frac{\mathrm{Ci}(\delta)}{\ln b}.
\]
The remaining $\mathrm{Ci}(\delta)$ term is an oscillatory residual whose
amplitude decays as $\delta$ grows. Since it does not contribute to the smooth
long-range envelope, we drop it and obtain
\begin{equation}
r_b(\delta)
\approx
\bar r_b(\delta)
:=
1-\frac{\ln\delta+\gamma}{\ln b}.
\label{eq:app_log_envelope}
\end{equation}
Reinstating the skew adds the constant
$c_{\mathrm{total}}(\pi/2)/(\sqrt{d}\,\ln b)$, where
$c_{\mathrm{total}}=\sum_m c_m$. This term shifts the expected score but is
independent of $\delta$, so it leaves the logarithmic decay rate unchanged.

\subsection{Bandwidth-Based Attention Bound}
\label{app:bandwidth_bound_proof}

Let $G=(V,E)$ and define its bandwidth as
\begin{equation}
B(G)
=
\min_{\pi:V\to\{1,\ldots,|V|\}}
\max_{(u,v)\in E}
|\pi(u)-\pi(v)|.
\label{eq:app_bandwidth_def}
\end{equation}
Thus, for any node ordering $\pi$, some edge satisfies
$|\pi(u)-\pi(v)|\ge B(G)$.
Assume each node contributes approximately $\tau$ tokens to the prompt. Then
an edge with stretch $|\pi(u)-\pi(v)|$ has separation
\[
\delta \approx \tau |\pi(u)-\pi(v)|.
\]
Therefore some edge has $\delta\gtrsim \tau B(G)$.
Using \Cref{eq:app_log_envelope},
\[
\bar r_b(\delta)
=
1-\frac{\ln\delta+\gamma}{\ln b},
\]
which decreases in $\delta$, the worst-case retention is bounded by
\begin{equation}
\bar r_b^*(G)
\lesssim
\bar r_b(\tau B(G))
=
1
-
\frac{\ln(\tau B(G))+\gamma}{\ln b} = 1
-
\frac{\ln B(G)}{\ln b}
-
\frac{\ln \tau}{\ln b}
-
\frac{\gamma}{\ln b}.
\label{eq:app_bandwidth_bound}
\end{equation}
which decomposes into topology-based, text-based, and bias
terms.

\paragraph{Examples.}
For common graph families on $n$ nodes,
\[
B(P_n)=1,
\qquad
B(C_n)=2,
\qquad
B(S_n)=\left\lceil\frac{n-1}{2}\right\rceil,
\qquad
B(K_n)=n-1.
\]
The star value follows because placing the center in the middle of the ordering minimizes the maximum distance to any leaf, while the complete-graph value follows because the first and last vertices in any ordering are adjacent. Substituting these values into \Cref{eq:app_bandwidth_bound} gives the examples reported in the main text. For instance, with $b=10^4$ and $\tau=50$,
\[
\bar r_b^*(P_n)
\approx
1-\frac{\ln 50+\gamma}{\ln 10^4}
\approx
0.51,
\]
while for a 20-node star, $B(S_{20})=10$, so
\[
\bar r_b^*(S_{20})
\approx
1-\frac{\ln(500)+\gamma}{\ln 10^4}
\approx
0.26.
\]
This illustrates how the same text length can produce substantially different attention retention depending on graph topology.

\section{Dataset Details}
In this section we explain the dataset details for the synthetic experiments, semi-synthetic experiments, and real-world experiments.

%%%%%%%%%%%%%%%%%%%%%%%%%%%%%%%%%%%%%%%%%%%%%%%%%%%%%%%%%%%%%%%%%%%%%%%%

%%%%%%%%%%%%%%%%%%%%%%%%%%%%%%%%%%%%%%%%%%%%%%%%%%%%%%%%%%%%%%%%%%%%%%%%

\subsection{Synthetic Validation Setup}
\label{sec:synth_graphs}

To empirically validate \Cref{thm:rope_edge_distortion}, we construct a controlled study designed to isolate the impact of positional edge stretch ($\delta_{u,v}$) on attention decay, independent of semantics.

\noindent \textbf{Graph Topologies.}
We generate graphs with $N=30$ nodes from three distinct families to capture diverse structural properties:
\begin{itemize}[leftmargin=*]
    \item \textbf{Erdős–Rényi (ER):} Models homogeneous random connectivity. We set the edge probability $p \approx \frac{\ln N + \ln \ln N}{N}$ to ensure graphs are connected with high probability without making them overly dense.
    \item \textbf{Stochastic Block Model (SBM):} We generate graphs with $k=3$ communities using an assortativity ratio of $p_{\mathrm{in}}/p_{\mathrm{out}} = 5.0$, testing the impact of inter-community edges on attention.
    \item \textbf{Barabási–Albert (BA):} Models scale-free networks with preferential attachment. To maintain comparability, we set the attachment parameter $m$ such that the average degree matches that of the ER graphs (i.e., $m \approx \lfloor p(N-1)/2 \rfloor$).
\end{itemize}

\vspace{0.1cm}
\noindent \textbf{Linearization Strategies.}
To vary the stretch distribution $\delta_{u,v}$, we serialize nodes using three orderings:
(1) \textbf{Random}, acting as a high-entropy baseline;
(2) \textbf{Breadth-First Search (BFS)}, which preserves local neighborhoods but incurs large jumps during backtracking; and
(3) \textbf{Reverse Cuthill–McKee (RCM)}, an algorithm explicitly designed to minimize matrix bandwidth. In our context, RCM serves as the theoretical minimizer of maximum edge stretch.

\vspace{0.1cm}
\noindent \textbf{De-biasing via Distance Balancing.}
A naive comparison of orderings is confounded by the fact that efficient orderings (like RCM) naturally produce shorter edge stretches than Random ordering. To decouple the effect of \textit{distance} from the effect of \textit{topology}, we employ a greedy balancing strategy during dataset generation. We sample query-target pairs $(u, v)$ such that the distribution of sequence distances $\delta_{u, v}$ is approximately uniform across all graph types and orderings. This ensures that any decay observed at a specific distance $D$ is attributable to the positional mechanism itself, rather than a lack of samples at that range.

\subsection{Semi-Synthetic and Real-World Datasets}
\label{sec:datasets_setup}

To rigorously evaluate the impact of graph serialization on LLM reasoning, we process three standard benchmarks into a format suitable for language modeling. Our data generation pipeline consists of three stages: sampling, linearization, and prompt construction.

\vspace{0.1cm}
\noindent \textbf{Subgraph Sampling.}
We use 1,000 ego nodes for testing and a disjoint 1,000 nodes for training fine-tuned baselines, with 200 training nodes reserved as GaLA's validation. For each ego node $u$, we extract a 2-hop subgraph and cap fan-out to fit standard LLM context limits: 20 neighbors per hop for semi-synthetic experiments and 5 for real-world experiments, sampling uniformly when nodes exceed this limit. In the real-world setting, we truncate each node's text to 50 words.

\vspace{0.1cm}
\noindent \textbf{Prompts.} \textit{Semi-Synth:}
For the ego node $u$, we mask its label token (replacing it with \texttt{<Masked>}). For all neighbor nodes $v \in \mathcal{N}(u)$, the prompt explicitly states their class label. The model is queried to predict the most common label among the ego node's immediate neighbors. This task forces the model to perform a retrieval-and-count operation over the serialized graph structure.

\begin{figure}[h]
\centering
\begin{tcolorbox}[colback=gray!5!white,colframe=gray!75!black,title=Semi-Synthetic Prompt Structure]
\small
\textbf{[System Message]} \\
You are a graph reasoning expert. Output ONLY the label name.

\vspace{0.2cm}
\textbf{[User Message]} \\
\textbf{\#\#\# Graph Context} \\
Node 1456 has label Theory and neighbors \{ 1095 \}. \\
Node 792 has label Theory and neighbors \{ 798, 1095 \}. \\
... \\
Node 1360 has label \textbf{<Masked>} and neighbors \{ 1095 \}.

\textbf{\#\#\# Task} \\
Query: What is the most common label across the neighbors of Node 1360? Choose one from: [Case Based, ... Theory].
\end{tcolorbox}
\vspace{-0.3cm}
\caption{A semi-synthetic prompt where node features are replaced by labels. The system instruction enforces a concise output format.}
\label{fig:semi_synth_prompt}
\end{figure}

\textit{Real World:}
Each node is associated with its raw text feature string. The model is queried to predict the class label of the ego node.

\begin{figure}[h]
\centering
\begin{tcolorbox}[colback=gray!5!white,colframe=gray!75!black,title=Real-World Prompt Structure]
\small
\textbf{[System Message]} \\
You are a graph reasoning expert. Output ONLY the label name.

\vspace{0.2cm}
\textbf{[User Message]} \\
\textbf{\#\#\# Graph Context} \\
Node 1433 has text Title: Learning Boxes... Abstract: We present exact learning algorithms... and neighbors \{ 1095, 798 \}. \\
... \\
Node 1360 has text Title: Learning From a Consistently Ignorant Teacher... and neighbors \{ 1095 \}.

\textbf{\#\#\# Task} \\
Query: Predict the label of Node 1360 based on its own text and graph neighborhood. Choose one from: [Case Based, ... Theory].
\end{tcolorbox}
\vspace{-0.3cm}
\caption{A real-world prompt using raw text attributes. The model must reason over both local text content and neighbor context, guided by the system instruction.}
\label{fig:real_world_prompt}
\end{figure}

\subsection{Chain-of-Thought (CoT) Prompting}
To evaluate whether explicit reasoning steps can overcome the decay induced by linearization, we introduce a third experimental setting employing Chain-of-Thought (CoT) prompting.
In this scenario, we modify the system instruction to encourage intermediate reasoning steps before the final prediction.
This setup tests the hypothesis that allowing the model to generate intermediate tokens might help it bridge the attention gaps caused by linearization, effectively refreshing the context of distant neighbors before making a decision. An example prompt is given in \Cref{fig:cot_prompt}, with the only difference being the system message.

\begin{figure}[h]
\centering
\begin{tcolorbox}[colback=gray!5!white,colframe=gray!75!black,title=Chain-of-Thought Prompt Structure]
\small
\textbf{[System Message]} \\
You are a graph reasoning expert. First, think step-by-step to analyze the graph structure and neighbors. Finally, state the label clearly starting with 'Answer:'.

\vspace{0.2cm}
\textbf{[User Message]} \\
\textit{(Same Graph Context as Figure~\ref{fig:real_world_prompt})}
\end{tcolorbox}
\vspace{-0.3cm}
\caption{The Chain-of-Thought (CoT) prompt structure. The modified system instruction explicitly prompts the model to generate reasoning traces before the final prediction.}
\label{fig:cot_prompt}
\end{figure}

\subsection{Ego-Node Placement Strategy}
Across GaLA settings, we place the ego node, i.e., the target of the classification query, at the end of the linearized sequence immediately before the query token. This aligns with the causal mask in decoder-only Transformers where a token at position $t$ can attend only to positions $1,\dots,t$, so placing the ego node earlier would prevent its hidden state from incorporating neighbors that appear later. %Placing it last lets the ego representation attend over the full serialized neighborhood, allowing it to act as an aggregation center before prediction, analogous to message passing in GNNs.

%%%%%%%%%%%%%%%%%%%%%%%%%%%%%%%%%%%%%%%%%%%%%%%%%%%%%%%%%%%%%%%%%%%%%%%%

%%%%%%%%%%%%%%%%%%%%%%%%%%%%%%%%%%%%%%%%%%%%%%%%%%%%%%%%%%%%%%%%%%%%%%%%
\section{Additional Experimental Modeling Details}

This section provides additional experimental details.

\subsection{Motivational Study: Controlled Aggregation}
\label{sec:cnat}

\vspace{0.1cm}
\noindent \textbf{Layer/Head Selection.}
We analyze attention only in heads and layers most relevant to semantic aggregation: the top-$k$ heads ($k=25\%$), selected dynamically by pre-RoPE attention strength, and late but non-final transformer layers. For Qwen2.5-3B-Instruct, this corresponds to layer 29, roughly 80\% through the model. 
\subsection{Node Classification}

Our evaluation compares GaLA against a set of baselines designed to test whether standard modeling strategies can mitigate linearization-induced degradation without architectural modification.

\noindent \textbf{Model Backbones.}
All model checkpoints are initialized from their Hugging Face repos.
%, with the explicit Model ID given below:

\begin{itemize}
    \item \textbf{Qwen2.5-3B-Instruct: } Qwen/Qwen2.5-3B-Instruct
    \item \textbf{Ministral-3B: } mistralai/Ministral-3-3B-Instruct-2512
    \item \textbf{Qwen2.5-7B-Instruct: } Qwen/Qwen2.5-7B-Instruct
\end{itemize}

%\noindent \textbf{Baselines and Interventions.}
%We consider two serialization strategies, random ordering and BFS, to isolate the effect of edge stretch induced by linearization alone. We exclude RCM from this main comparison as preliminary experiments showed it consistently underperforms, likely due to the disruption of local semantic neighborhoods despite global bandwidth minimization. To test whether explicit reasoning can compensate for positional decay, we additionally evaluate Chain-of-Thought prompting under two token budgets (256 and 512 tokens). Chain-of-Thought provides an understanding of whether more compute can enable the model to re-attend to the correct neighbors. Finally, as a strong adaptation baseline, we include fine-tuning, serving as an upper bound that reflects what can be achieved through training rather than inference-time adjustments. For the Qwen2.5-3B-Instruct and Ministral-3B we perform full fine-tuning, while for Qwen2.5-7B-Instruct we use Low-Rank Adaptation (LoRA) to maintain computational efficiency.
When fine-tuning, we use the 1,000 node training for one epoch with a learning rate of $1 \times 10^{-5}$, a batch size of $1$, and gradient accumulation of $8$. The max sequence length is set to $2048$. 

\noindent \textbf{GaLA Configuration and Tuning.}
For all GaLA experiments, the scalar bias strength $\beta$ is selected using a small validation set of 200 samples disjoint from the test data. This setup reflects a realistic low-resource calibration scenario and avoids reliance on tuning.

\section{Additional Results and Analyses}
\label{sec:addtl_exp}
Here we provide additional sensitivity studies to further characterize the behavior and robustness of GaLA.
\paragraph{Head-selection strategy.}
\Cref{tab:headsel_ablation} ablates which heads receive the graph bias. We study the following cases: every head (\emph{All}), a \emph{random} subset, the \emph{entropy}-ranked heads, or the \emph{gradient} score. We find that nearly every choice improves over the BFS baseline on the semi-synthetic task, and thus the benefit is not tied to a single heuristic. Furthermore, the label-free entropy rule is already a strong default, and the the gradient score, which reuses the small calibration set, is best or tied throughout and yields the largest gains on the harder PubMed and Arxiv tasks. We see similar, yet slightly smaller gains, on the real-world tasks, similar to that found in the main text.

\begin{table}[t]
    \centering
    \renewcommand{\arraystretch}{1.05}
    \setlength{\tabcolsep}{3.5pt}
    \caption{\textbf{Head-selection ablation.} On Qwen2.5-3B, gradient-based selection performs best or tied on most tasks, while entropy provides a strong label-free alternative.}
    \label{tab:headsel_ablation}
    \vspace{-0.2cm}
    \footnotesize
    \begin{tabular}{l l ccc ccc}
    \toprule
    & & \multicolumn{3}{c}{\textbf{Semi-Synth. Perf.}}
    & \multicolumn{3}{c}{\textbf{Real-World Perf.}} \\
    \cmidrule(lr){3-5} \cmidrule(lr){6-8}
    \textbf{Method}
    & \textbf{Head selection}
    & \textbf{Cora} & \textbf{PubMed} & \textbf{Arxiv}
    & \textbf{Cora} & \textbf{PubMed} & \textbf{Arxiv} \\
    \midrule
    BFS Lin. & --
        & $83.3$ & $90.3$ & $80.3$
        & $65.9$ & $77.5$ & $35.2$ \\
    \midrule
    \multirow{4}{*}{GaLA}
        & All
        & $89.1$ & $91.2$ & $82.4$
        & $67.7$ & $80.1$ & $35.2$ \\
        & Random
        & $87.8$ & $90.9$ & $83.1$
        & $66.8$ & $80.2$ & $37.0$ \\
        & Entropy
        & $89.1$ & $91.5$ & $83.1$
        & $67.0$ & $79.2$ & $36.0$ \\
        & Gradient
        & $90.4$ & $92.0$ & $82.4$
        & $67.4$ & $82.5$ & $36.8$ \\
    \bottomrule
    \end{tabular}
\end{table}

\paragraph{Runtime overhead.}
\Cref{tab:overhead} decomposes GaLA's cost on Qwen2.5-3B. We find that runtime is largely dominated by ordinary LLM inference, i.e. the per-sample shortest-path computation adds at most tens of milliseconds, and the one-time head-selection calibration (entropy or gradient) costs only $\sim$$30$ seconds, which is negligible once amortized over the test set. Consequently, the complete $n{=}1000$ pipeline stays close to baseline inference and orders of magnitude below Chain-of-Thought (\Cref{fig:runtime}), making GaLA practical to deploy.

\begin{table}[t]
    \centering
    \renewcommand{\arraystretch}{1.15}
    \setlength{\tabcolsep}{5.0pt}
    \caption{\textbf{GaLA Runtime Breakdown.} For Qwen2.5-3B, GaLA runtime is dominated by LLM inference. Shortest-path computation adds little per sample, while entropy and gradient calibration are one-time costs that remain small relative to the full $n{=}1000$ pipeline.}
    \label{tab:overhead}
    \vspace{-0.2cm}
    \footnotesize
    \begin{tabular}{l ccc}
    \toprule
    \textbf{Component} & \textbf{Cora} & \textbf{PubMed} & \textbf{Arxiv} \\
    \midrule
    Inference, GaLA (s / data point) 
        & $0.232$ & $0.438$ & $0.822$ \\
    \quad \textit{of which}: shortest-path comp. (s) 
        & $0.001$ & $0.013$ & $0.066$ \\
    \midrule
    \multicolumn{4}{l}{\textit{Pre-inference calibration, one-time (s)}} \\
    \quad Entropy 
        & $30$ & $35$ & $33$ \\
    \quad Gradient 
        & $39$ & $28$ & $29$ \\
    \midrule
    \multicolumn{4}{l}{\textit{Complete pipeline, $n{=}1000$ (s)}} \\
    \quad Entropy 
        & $262$ & $473$ & $855$ \\
    \quad Gradient 
        & $271$ & $466$ & $851$ \\
    \bottomrule
    \end{tabular}
\end{table}

\end{document}